\documentclass[accepted]{uai2022} % for initial submission
\usepackage[american]{babel}
\usepackage{natbib} % has a nice set of citation styles and commands
\usepackage{mathtools} % amsmath with fixes and additions
\usepackage{booktabs} % commands to create good-looking tables
\usepackage{tikz} % nice language for creating drawings and diagrams

\usepackage{microtype}
\usepackage{graphicx}
\usepackage{algorithm}
\usepackage{algorithmic}
\usepackage{booktabs} % for professional tables
\usepackage{amsmath}
\usepackage{amsthm}
\usepackage{amssymb} % triangleq
\usepackage{color, colortbl}
\usepackage{xcolor}
\usepackage{hyperref}
\usepackage{cleveref}

\usepackage{bm}
\usepackage{bbm}
\usepackage{graphicx}
\usepackage{thmtools}
\usepackage{thm-restate}
\usepackage{mathtools}
\usepackage[skip=0pt]{caption}
\usepackage[skip=0pt]{subcaption}
\usepackage{booktabs} % toprule
\usepackage{fontawesome}
\definecolor{red}{HTML}{E51400}  %red
\definecolor{blue}{HTML}{0050EF} %cobalt
\definecolor{green}{HTML}{008A00} %emerald
\definecolor{purple}{HTML}{AA00FF} %violet
\definecolor{dark-red}{rgb}{0.4, 0.15, 0.15}
\definecolor{dark-blue}{rgb}{0.15, 0.15, 0.4}
\definecolor{medium-red}{rgb}{0.5, 0, 0}
\definecolor{medium-blue}{rgb}{0, 0, 0.5}
\definecolor{light-red}{rgb}{0.7, 0, 0}
\definecolor{light-blue}{rgb}{0, 0, 0.7}

\newtheorem{lemma}{\bf Lemma}

\newtheorem{proposition}{\bf Proposition}
\newtheorem{definition}{\bf Definition}

\newtheorem{assumption}{\bf Assumption}

\definecolor{red}{HTML}{E51400} %red
\definecolor{blue}{HTML}{0050EF} %cobalt
\definecolor{green}{HTML}{008A00} %emerald
\definecolor{purple}{HTML}{AA00FF} %violet
\definecolor{orange}{HTML}{FF7F00}
\definecolor{gray}{HTML}{848482}

\DeclareMathOperator*{\argmax}{arg\,max}

\newcommand{\E}{\mathbb{E}}

\newcommand{\R}{\mathbb{R}}
\newcommand{\bT}{\mathbb{T}}

\newcommand{\bA}{\boldsymbol{A}}

\newcommand{\bb}{\boldsymbol{b}}
\newcommand{\bB}{\boldsymbol{B}}
\newcommand{\bC}{\boldsymbol{C}}
\newcommand{\bD}{\boldsymbol{D}}

\newcommand{\bH}{\boldsymbol{H}}
\newcommand{\bh}{\boldsymbol{h}}
\newcommand{\bI}{\boldsymbol{I}}

\newcommand{\bM}{\boldsymbol{M}}
\newcommand{\bN}{\boldsymbol{N}}

\newcommand{\bS}{\boldsymbol{S}}
\renewcommand{\bT}{\boldsymbol{T}}

\newcommand{\bu}{\boldsymbol{u}}
\newcommand{\bV}{\boldsymbol{V}}

\newcommand{\bx}{\boldsymbol{x}}

\newcommand{\cB}{\mathcal{B}}
\newcommand{\cC}{\mathcal{C}}
\newcommand{\cD}{\mathcal{D}}
\newcommand{\cE}{\mathcal{E}}

\newcommand{\cG}{\mathcal{G}}
\newcommand{\cH}{\mathcal{H}}

\newcommand{\cN}{\mathcal{N}}
\newcommand{\cQ}{\mathcal{Q}}

\newcommand{\cT}{\mathcal{T}}
\newcommand{\cU}{\mathcal{U}}

\newcommand{\btheta}{\boldsymbol{\theta}}

\newcommand{\norm}[1]{\left\lVert #1\right\rVert}
\newcommand{\prdct}[1]{\left\langle #1\right\rangle}

\newcommand{\compilefullversion}{true}%SHOW full version
\ifthenelse{\equal{\compilefullversion}{false}}{%
	\newcommand{\OnlyInFull}[1]{}
	\newcommand{\OnlyInShort}[1]{#1}
}{%
	\newcommand{\OnlyInFull}[1]{#1}%
	\newcommand{\OnlyInShort}[1]{}%
}%

\newcommand{\compilehidecomments}{true}%HIDE comments
\ifthenelse{ \equal{\compilehidecomments}{true} }{%
	\newcommand{\wei}[1]{}
	\newcommand{\xutong}[1]{}
	\newcommand{\jinhang}[1]{}
	\newcommand{\tong}[1]{}
}{
\newcommand{\wei}[1]{{\color{blue}{\small{\bf [Wei: #1]}}}}
\newcommand{\xutong}[1]{{\color{green} [#1]}}

\newcommand{\tong}[1]{{\color{blue}{\small{\bf [Tong: #1]}}}}
}



\title{Federated Online Clustering of Bandits}

\author[1]{\href{mailto:<liuxt@cse.cuhk.edu.hk>}{Xutong Liu}}
\author[2]{\href{mailto:<zhaohaoru@sjtu.edu.cn>}{Haoru Zhao}}
\author[3]{\href{mailto:<tyu@adobe.com>}{Tong Yu}}
\author[2]{\href{mailto:<shuaili8@sjtu.edu.cn>}{Shuai Li\thanks{Correspondence to: Shuai Li <shuaili8@sjtu.edu.cn>}}}
\author[1]{\href{mailto:<cslui@cse.cuhk.edu.hk>}{John C.S. Lui}}
\affil[1]{%
    The Chinese University of Hong Kong\\
    Hong Kong SAR, China
}
\affil[2]{%
    Shanghai Jiao Tong University\\
    Shanghai, China
}
\affil[3]{%
    Adobe Research\\
    San Jose, CA, USA
  }

\begin{document}
\maketitle

\begin{abstract}
% Contextual multi-armed bandit (MAB) is an important sequential decision making problem in recommendation systems. A line of works, called the clustering of bandits (CLUB), utilizes the collaborative effect over users and dramatically improves the recommendation quality. Owing to the increasing scale of application and public concerns of privacy, there is a growing demand to keep user data decentralized and push bandit learning to the local server side. Existing CLUB algorithms, however, are designed under the centralized setting where data are available at a central server. We focus on the study of the federated clustering of bandit (FCLUB) problem, whose goal is to minimize the total regret meanwhile satisfying privacy and communication considerations. \tong{To achieve this, what is the UNIQUE challenge. May mention Cluster Structure Detection.} For this problem, we design a new phase-based scheme for cluster detection and a novel asynchronous communication protocol for cooperative bandit learning. To protect users' privacy, previous differential privacy (DP) definitions are not very suitable \tong{elaborate why} and we propose a new DP notion which acts on the user cluster level. We provide rigorous proofs to show that our algorithm simultaneously achieves (clustered) DP, sublinear communication complexity and sublinear regret. Finally, experimental evaluations show our superior performance compared with benchmark algorithms.

Contextual multi-armed bandit (MAB) is an important sequential decision-making problem in recommendation systems. A line of works, called the clustering of bandits (CLUB), utilize the collaborative effect over users and dramatically improve the recommendation quality. Owing to the increasing application scale and public concerns about privacy, there is a growing demand to keep user data decentralized and push bandit learning to the local server side. Existing CLUB algorithms, however, are designed under the centralized setting where data are available at a central server. We focus on studying the federated online clustering of bandit (FCLUB) problem, which aims to minimize the total regret while satisfying privacy and communication considerations. We design a new phase-based scheme for cluster detection and a novel asynchronous communication protocol for cooperative bandit learning for this problem. To protect users' privacy, previous differential privacy (DP) definitions are not very suitable, and we propose a new DP notion that acts on the user cluster level. We provide rigorous proofs to show that our algorithm simultaneously achieves (clustered) DP, sublinear communication complexity and sublinear regret. Finally, experimental evaluations show our superior performance compared with benchmark algorithms.

\end{abstract}

\section{Introduction}
Stochastic multi-armed bandit (MAB) \citep{auer2002finite} is a well-known sequential decision-making problem, where a learner sequentially selects actions so as to maximize the cumulative rewards (or minimize the cumulative regret). 
One fruitful application area of MAB is the online recommendation systems (RecSys)~\citep{chu2011contextual, abbasi2011improved, gentile2014online,li2019improved, zhang2020conversational, li2021unifying},
where MAB algorithms provide
a principled way to handle the challenge of exploration-exploitation trade-off~\citep{lattimore2020bandit}.

To advance the bandit algorithm for large-scale applications, contextual linear bandits add the simple yet effective linear structure assumptions on actions and reward functions~\citep{chu2011contextual,li2010contextual, abbasi2011improved}.
One limitation, however, is that such a model mainly works in a content-dependent manner, ignoring the often used tool of collaborative filtering.
To address this issue, the clustering of bandits (CLUB) are proposed~\citep{gentile2014online, li2016collaborative, li2018online,li2019improved}.
The CLUB algorithms adaptively cluster similar users and utilize the collaborative information given by the cluster structure, which dramatically improves the recommendation quality.

While most existing bandit algorithms are designed under a centralized setting,
in response to the increasing application scale and public concerns about privacy, there is a growing demand to keep user data decentralized and push the learning of bandit models to the client or the local server side \tong{This claim is very strong. Please cite. kairouz2021advances is not a bandit paper. Can we cite papers to support this claim in bandit papers?}. 
This paradigm is now known as federated learning~\citep{kairouz2021advances}.
Owing to its overall applicability, there has been a surge of interest in studying federated MAB~\citep{dubey2020differentially, zhu2021federated,shi2021federated}, which promises cooperative bandit learning with larger amounts of data (across multiple local servers) while keeping the data decentralized.
This motivates us to study the CLUB problem to its federated counterpart, i.e., the federated clustering of bandits (FCLUB).

In FCLUB, each local server can conduct its own local clustering of bandit algorithms.
To enable the collaborative effects of users across different servers, the local server could also collaborate with other local servers under the coordination of a global server, whose communication needs to satisfy specific privacy and communication requirements.
The goal of this work is to design an federated online clustering of bandit framework, so as to minimize the $T$-round regret under the privacy protection requirements and communication cost considerations.

The key challenge of FCLUB is designing collaborative bandit learning procedures and cluster detection strategies to identify the overall cluster structures across different local servers, where each local server only holds part of the users with unknown interests.
Such a problem is more challenging due to the following privacy and communication cost requirements, which are two first-order requirements for any federated applications~\citep{kairouz2021advances}.

\textbf{Privacy protection:} To reduce the privacy leakage of each user, we expect local servers to only share user clusters' data instead of individuals' raw data.
In addition, we still need a mechanism to protect the uploaded (cluster) information against possible adversaries outside the local server, for which we adopt the solution concept of differential privacy (DP).
However, the off-the-shelf DP notion is defined on individual users, hence unsuitable for FCLUB.
It is challenging and unclear what is a suitable notion of privacy over the clustering of users and how to devise algorithms to guarantee the corresponding privacy requirements.

\textbf{Communication:} Communication is critical for collaborative learning, but may also be expensive or time-consuming. For FCLUB, it is desired to minimize the total regret while keeping the communication costs (in terms of communication rounds between the global server and local servers) as low as possible.
Another requirement is to design an asynchronous communication protocol incorporating the randomly arriving users and possibly lagging servers, preventing commonly used synchronous protocols~\citep{dubey2020differentially}.

\tong{Although this paragraph seems to describe the limitations of CLUB, the sentences below are still general and not uniquely related to CLUB. May improve by adding words `cluster' in suitable places of these sentences. E.g., cluster detection procedure is mentioned above once, but never mentioned again in this paragraph, which make all pieces a bit disconnected.}

\subsection{Our Contributions}
To address the aforementioned challenges, this paper makes four contributions.

\textbf{1. Problem Formulation:} We propose the setting of online clustering of bandits to its federated counterpart, which considers the privacy protection and communication requirements.
We also propose a novel cluster differential privacy (CDP) notion tailored for the FCLUB setting.

\textbf{2. Algorithm Design:}
\tong{Cluster Structure Detection seems to be the key challenge, but is missing here. Which parts of our methodologies are motivated by this challenge?}
We propose a private and communication-efficient FCLUB-CDP algorithm. For privacy protection, a tree-based privatizer is designed to guarantee our proposed CDP.
For communication efficiency, we follow the phase-based principle for cluster detection and propose the asynchronous communication protocol for delayed information sharing.
In particular, each local server maintains upload/download buffers and occasionally uploads/downloads the buffered information to/from the global server only if it finds the latest information deviates too far from the last update.

\textbf{3. Theoretical Analysis:} We prove that FCLUB-CDP achieves the $O(dL\sqrt{mT\frac{\log(1/\delta)}{\varepsilon}}\log^{1.5} T)$ regret bounds, $O(dmL\log T)$ communication costs and $(\varepsilon, \delta, L,m)$-CDP privacy guarantee, respectively. 

\textbf{4. Experiments:} We conduct extensive experiments over synthetic and real-world datasets to validate our theoretical analysis. Empirical results show the superior performance of our algorithm over existing algorithms.\footnote{Codes and datasets are available at  \href{https://github.com/ZhaoHaoRu/Federated-Clustering-of-Bandits}{GitHub}.}

\subsection{Related Work}\label{sec: related work}
\textbf{Online Clustering of Bandits.} The online clustering bandits is first proposed by \citet{gentile2014online} and shows its effeteness by accelerating the learning process of contextual bandits. The key idea is to use a graph representing the user similarity and adaptively refine the user clusters for information sharing.
This work has been extended by a series of works considering the collaborative effects on both users and items~\citep{li2016collaborative}, the context-aware settings~\citep{gentile2017context}, the cascading bandit setting~\citep{li2018online} and the users with different user frequency~\citep{li2019improved}.
However, none of these works consider the privacy constraints and communication cost requirements imposed by the FL paradigm like the current work, and therefore cannot give guarantees on these two critical criteria.
\citet{korda2016distributed} considers the peer-to-peer but non-private clustering of bandits, our work studies the private bandit setting under the orchestration of a global server, which requires different algorithms and analysis.

% These 
% Li2016a extends first proposed the online clusting of linear bandits algorithm which maintains a graph among users and uses connected components to represent user clusters.
% And there are a series of works proposing several variants of this algorithm, such as \cite{li2018online},\cite{li2019improved} and \cite{gentile2017context}. In this paper, we use this algorithm to process the users on local servers, which means we maintain a graph on every local server.

\textbf{Federated and Distributed Bandits.} There has been growing interest in bandit learning with multiple players.
One line of research investigates the competitive agents with collisions~\citep{anandkumar2011distributed,rosenski2016multi,bistritz2018distributed,boursier2019sic}, in which the reward for an arm is zero if it is chosen by more than one agent. The goal of these works is to minimize regret without communication, which is different from ours.
The cooperative distributed bandits are most related to our work, in which multiple agents collaborate to solve a bandit problem over certain communication networks, e.g., peer-to-peer networks~\citep{korda2016distributed} or client-server networks~\citep{dubey2020differentially,li2021asynchronous}.
Our work belongs to the client-server setting, but we differ from both \citet{dubey2020differentially} and  \citet{li2021asynchronous} since neither of them considers the clustering effects of users.
% Our work belongs to the client-server setting but we differ from \cite{dubey2020differentially} which assumes users are identical and comes in a round-robin manner or \cite{li2021asynchronous} which does not consider the clustering of users and differential privacy.

\textbf{Differential Privacy.} Our work leverages on \textit{differential privacy}, a rigorous mathematical framework of privacy first proposed by \citet{dwork2006calibrating}.
% Encoding the intuition that any observable output changes very little (in probability) when any input datum changes, differential privacy has been accepted as the \textit{de facto} gold standard of privacy-preserving data analysis in both academia and industry.
We utilize several useful techniques from the standard differential privacy to maintain our cluster differential privacy condition.
Most notably, we use a tree-based algorithm
% \shuai{this should be put in the discussion part}
which is introduced in \citet{chan2011private} to realize differential privacy for the continual release of statistics.
% This algorithm can add $\log n$ noisy terms to protect the partial sums of $n$ entries by maintaining a binary tree whose nodes correspond to these $n$ entries.
In the single-agent bandit setting, \citet{shariff2018differentially} also utilizes this tree-based algorithm to achieve Joint DP, which is then extended by \citet{dubey2020differentially} to the federated setting.
The closest work to ours is \citet{dubey2020differentially} and they study the simpler case where each local server only holds one user and all users are identical (with the same unknown preference vector), hence gives the different user-level DP definition with different privatizer and analysis.

\textit{To the best of our knowledge, this paper is the first to generalize the CLUB to its federated setting, which simultaneously achieves privacy protection and communication requirements.}

\section{Problem Settings}\label{sec: problem setting}

In this section, we formulate the setting of ``Federated Clustering of Bandits'' (FCLUB). We use $[n]$ to represent set $\{1,...,n\}$. We use boldface lowercase letters and boldface capitalized letters for column vectors and matrices, respectively. For the norms, $\norm{\bx}$ denotes the $\ell_2$ norm of vector $\bx$. For any symmetric positive semi-definite (PSD) matrix $\bM$ (i.e., $\bx^{\top} \bM \bx \ge 0, \forall \bx$), $\norm{\bx}_{\bM}=\sqrt{\bx^{\top} \bM \bx}$ denotes the matrix norm of $\bx$ regarding matrix $\bM$.

At the global level, there are $n$ users, denoted by the set $\cU=\{u_1, ..., u_n\}$.
Each user $i \in \cU$ has an \textit{unknown} preference vector $\btheta_i \in \R^d$ and for simplicity we assume $\norm{\btheta_i}\le 1$.
Since users may have the same/similar preference vector, we assume there exists $m$ (unknown) different preference vectors, i.e., $|\{\btheta_1, ..., \btheta_n\}|=m$.
Users with the same preference vector form an underlying cluster and we denote these $m$ (unrevealed) clusters by $\cC=\{C_1, ..., C_m\}$.
Different from CLUB, users in FCLUB are distributed in $L$ local servers denoted by $\{1,...,L\}$.
At the local level, the local server $\ell$ contains $n^\ell$ users $\cU^\ell=\{u^{\ell}_1,..., u^{\ell}_{n^\ell}\}$ (with $\bigcup_{\ell \in [L]}\cU^\ell=\cU$) and similarly, these $n^\ell$ users form local cluster $\cC^\ell=\{C^\ell_1, ..., C^\ell_{m_{\ell}}\}$ where $m_{\ell} \le m$.

The learning agent interacts with the bandit game as the follows.
At each time $t$, a user $i_t \in [n]$ randomly arrives with probability $1/n$.
Then $K$ items are generated to form a item set $\bD_t$, where the feature of each item $\bx \in \bD_t$ is drawn independently from a fixed but unknown distribution $\rho$ over $\{\bx \in \R^d: \norm{\bx}\le 1\}$.
The learning agent identifies the local server $\ell_t$ that $i_t$ belongs to and the current user cluster $j_t$ (detected by our algorithm) which $i_t$ lies in.
The local server then recommends an item $\bx_{t} \in \bD_t$ to the user based on the aggregated information from cluster $j_t$.
After $i_t$ receives the recommendation, the learning agent receives a random reward $y_t \in [0,1]$.
Let $\cH_t=\{i_1, \bx_1, y_1,..., i_{t-1}, \bx_{t-1}, y_{t-1}, i_t\}$ be the historical information before time $t$.
We assume the expectation of reward $y_t$ is linear in the feature vector $\bx \in \bD_t$ and the unknown preference vector $\btheta_{i_t}$, i.e., $\E_t[y_t|\bx]=\btheta_{i_t}^{\top}\bx$, and $\{y_t-\btheta_{i_t}^{\top}\bx\}_{t=1,2,...}$ have sub-Gaussian tails $\sigma_0^2$. 

Now we give some assumptions on preference vectors and item feature vectors. Note that all the assumptions follow the previous works~\cite{gentile2014online,gentile2017context,li2018online, li2019improved}.
\begin{assumption}[Gap between preference vectors]
For any two different preference vectors $\btheta_{i_1} \neq \btheta_{i_2}$, there is a fixed but unknown gap $\gamma>0$ so that $\norm{\btheta_{i_1}-\btheta_{i_2}} \ge \gamma$.\footnote{As previous works, this assumption can be relaxed by assuming the existence of two thresholds, one for the between-cluster distance $\gamma$, the other for the within-cluster distance $\norm{\btheta_{i_1}-\btheta_{i_2}}\le\eta$. }
\end{assumption}

\begin{assumption}[Item regularity]
For item distribution $\rho$, there exists a known $\lambda_x > 0$ so that $\E_{\bx \sim \rho}[\bx \bx^{\top}]$ is full rank with minimal eigenvalue $\lambda_x$. Meanwhile, for all time $t$, for any fixed unit vector $\btheta \in \R^d$, $(\btheta^{\top} \bx)^2$ has sub-Gaussian tail with variance $\sigma^2 \le \frac{\lambda_x}{8 \log (4K)}$.
\end{assumption}

\textbf{Learning Efficiency.} The goal of the learning agent is to accumulate as much reward as possible.
Let the optimal item for user $i_t$ at time $t$ be $\bx^*_{i_t}=\argmax_{\bx \in \bD_t} \btheta_{i_t}^{\top}\bx$.
The learning performance is measured by the regret, defined as 
\begin{equation}
    R(T)=\E[\sum_{t=1}^T r_t]= \E[\sum_{t=1}^T(\btheta_{i_t}^{\top}\bx^*_{i_t}-\btheta_{i_t}^{\top}\bx_t)],
\end{equation}
where $r_t$ is the regret at time $t$ and the expectation is taken over the randomness of the algorithm and the environment regarding the users $i_1, ..., i_T$ and the item sets $D_1, ..., D_T$.

In addition to the regret, privacy protection and the communication cost are two important criterion in federated learning.
In this work, we aim to ensure that the user data are protected under privacy constraints and the communication complexity is low.

\textbf{Privacy Requirements.}
To protect the user data, we introduce two privacy requirements.
First, we desire the local server only uploads the user clusters' sufficient statistics (or clustered data) for the learning procedure, instead of individual users' raw data.
Second, to protect the clustered data against third-party adversaries outside the local server, we adopt the notion of DP, which encodes the intuition that any observable output changes very little (in probability) when any input datum changes.
Since existing DP notions are defined on \textit{user-level} data~\cite{shariff2018differentially, dubey2020differentially}, we introduce a new differential privacy requirement to protect the \textit{cluster-level} data.

% \begin{definition}[Locally Differential Private (LDP)] A data releasing mechanism $\cQ \rightarrow \mathcal{Z}$ is defined as $(\varepsilon, \delta)$-LDP, if for any two outcomes $y, y' \in \cD$, and any (measurable) subset $O \subseteq \mathcal{Z}$, there is 
% \begin{equation}
%     \Pr[\cQ(y)\in U] \le \exp(\varepsilon)\Pr[\cQ(y)\in U] + \delta.
% \end{equation}
% \end{definition}
The contextual MAB problem involves two sets of variables against the adversaries outside the local server: the decision sets $\bD_t$ and the observed rewards $y_t$.
Since the users only receive and store observations regarding the chosen action $\bx_t$ and the observed reward $y_t$, it suffices to protect $(\bx_t,y_t)_{t\in [T]}$ to achieve DP requirements.
Let $\tau_{\ell,j}$ be time slots when user in cluster $j$ appears at local server ${\ell}$, we denote two sequences $S_{\ell,j}=(\bx_{t},y_{t})_{t \in \tau_{\ell,j}}$ and $S_{\ell,j}'=(\bx'_{t},y'_t)_{t \in \tau_{\ell,j}}$ as $t$-neighboring if $(\bx_{\tau},y_{\tau})=(\bx'_{\tau},y'_{\tau})$ for $\tau \neq t \in \tau_{\ell,j}$. 

\begin{definition}[Cluster Differential Privacy]\label{def:cdp} In the FCLUB setting with $L$ servers and (at most) $m$ clusters, a federated contextual bandit algorithm $A=(A_{\ell,j})_{\ell\in [L], j \in [m]}$ is $(\varepsilon, \delta, L, m)$-CDP, if for any $(\ell,j), (\ell',j')$ s.t. $(\ell,j)\neq (\ell',j')$, any $t$ and the set of sequences $\bar{S}_{\ell',j'} = \cup_{i \in [L], k \in [m]}S_{i,k}$ and $\bar{S}'_{\ell',j'}=\cup_{i \in [L]\backslash \ell', k \in [m]\backslash j'}S_{i,k} \cup S_{\ell',j'}'$ s.t. $S_{\ell',j'}$ and $S'_{\ell',j'}$ are $t$ neighboring, and for any subset of actions $a_{\ell,j} \subset \prod_{\tau \in \tau_{\ell,j}}\cD_{\tau}$ of actions, it holds that
\begin{equation}
\Pr[A_{\ell,j}(\bar{S}_{\ell',j'})\in a_{\ell,j} ] \le e^{\varepsilon}\Pr[A_{\ell,j}(\bar{S}'_{\ell',j'})\in a_{\ell,j}] + \delta.
\end{equation}
\end{definition}
Note that our CDP notion formalizes the intuition that the action chosen by any local server $\ell$ (at the cluster level $j$) must be sufficiently indistinguishable (in probability) to any single $(x,y)$ pair from any other local cluster $(\ell',j')$. Such a notion does not require each cluster is private to its own observations, i.e., each cluster of users can be trusted with its own data, which is different from the local DP~\citep{zheng2020locally} or Joint DP \citep{shariff2018differentially} that assume even itself cannot be trusted. 

\textbf{Communication Complexity.}
To evaluate the communication complexity, we count one upload operation (or one download operation) between any local server and the global server as one communication round. 
Our communication complexity is the total number of communication rounds $C(T)$ over the time horizon $T$.

\section{Algorithm}\label{sec:algorithm}
\begin{figure}
    \centering
    \includegraphics[width=0.5\textwidth]{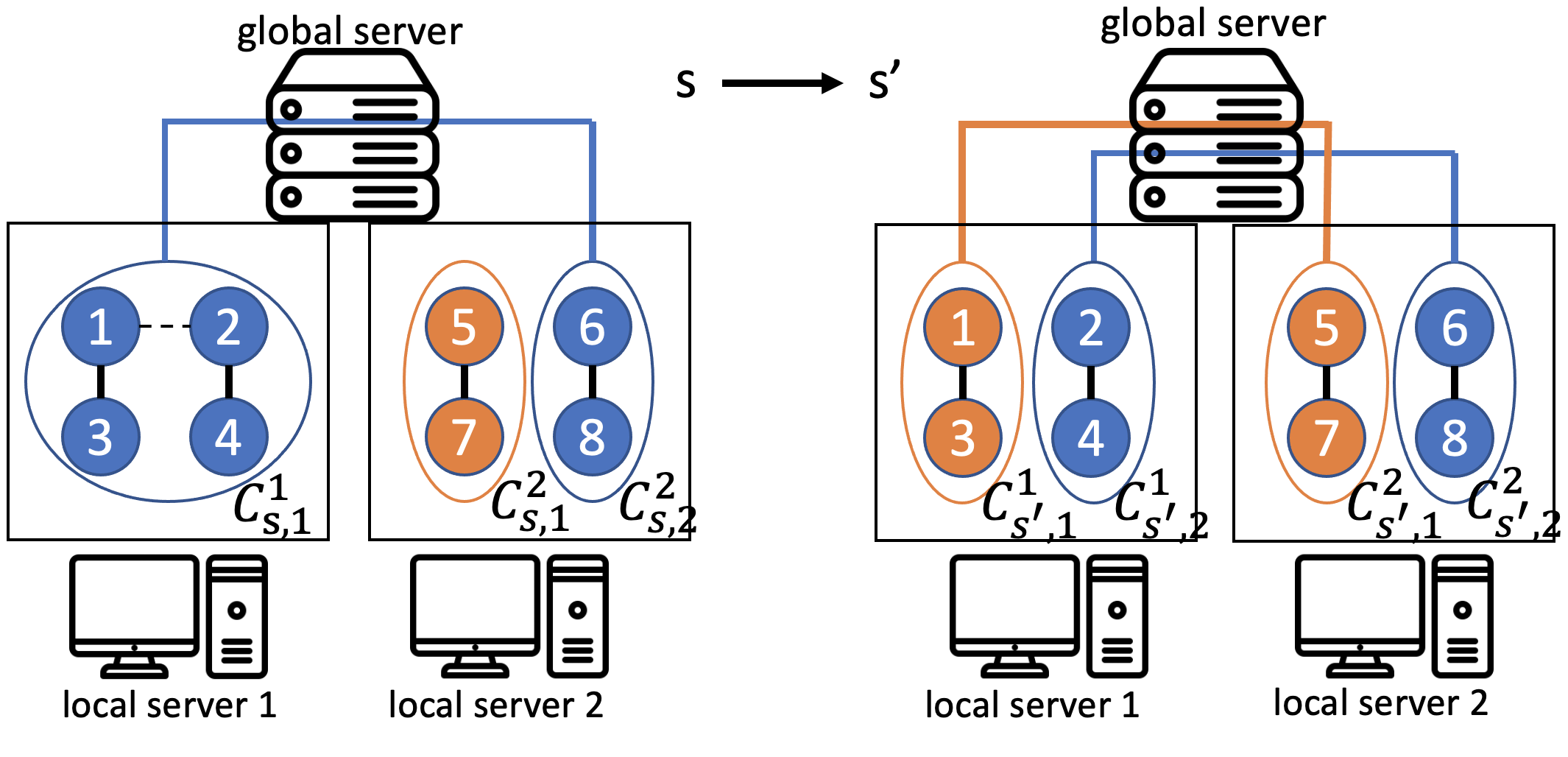}
    \caption{Illustration of how our algorithm detects clusters from phase $s$ to $s'$. Local server $1$ delete edge $(1,2)$ and split $C_{s,1}^{1}$ into $C_{s',1}^1, C_{s',2}^1$. The global server merges local clusters as $\{(C^1_{s',1}, C^2_{s',1}), (C^1_{s',2}, C^2_{s',2})\}$.}
    \label{fig:split_n_merge}
\end{figure}
In this section, we introduce our phase-based federated clustering of bandit algorithms with CDP (FCLUB-CDP).

\textbf{Identify the Underlying Cluster Structure.}
To correctly identify the cluster structure, we design a phase-based clustering detection algorithm in \Cref{alg:cluster_detect}.
The high level idea is to first conduct local-level clustering of bandits on each local server and merge the local clusters on the global server.

At the local level, each server $\ell$ maintains a profile of $(\bV_{t,i}, \bb_{t,i}, T_{t,i})$ for its own local users $i \in \cU^{\ell}$, where $V_{t,i}$ is Gramian matrix, $\bb_{t,i}$ is the moment vector of regressand by the regressors, and $T_{t,i}$ is the number of times that $i$ has appeared up to time $t$.
At the beginning of $s$ phase (or $t=2^s+1$), based on profiles $(\bV_{t,i}, \bb_{t,i}, T_{t,i})_{i \in \cU^{\ell}}$, each server $\ell$ maintains an undirected graph structure $\cG^{\ell}_s$, where nodes represent all local users $\cU^{\ell}$ and a pair of users are connected by an edge if they are similar.
We initialize the graph by a complete graph $\cG^\ell_{0}$ and gradually delete edges at every beginning of each phase $s$.
Specifically, in line~\ref{line:detect_delete}, we delete edge between any user $i_1$ and $i_2$ if the distance between their estimated preference vector are larger than the following threshold.
\begin{equation}\label{eq:edge_deletion}
    \norm{\hat{\btheta}_{t, i_1}-\hat{\btheta}_{t, i_2}} > \alpha_1 (F(T_{t, i_1})+F(T_{t,i_2})), 
\end{equation}
where $\hat{\btheta}_{t, i}=(\lambda \bI + \bV_{t, i})^{-1}\bb_{t,i}$ and $F(x)=\sqrt{\frac{1+\ln (1+x)}{(1+x)}}$.
After the deletion, users in connected components $j \in C(\cG^{\ell}_s)$ are grouped into local cluster $j$.
In line~\ref{line:detect_cluster}, server $\ell$ uploads the clustered information $I_{s,{\ell}}=(C^{\ell}_{s,j}, \tilde{\bV}^{\ell}_{s, j}, \tilde{\bb}^{\ell}_{s, j}, \tilde{T}^{\ell}_{s,j})_{j \in C(\cG^{\ell}_{s})}$ to the global server, which contains clustered information  for each cluster $j \in C(\cG^{\ell}_s)$.
Note that $I_{s,\ell}$ are added with random perturbation to protect the users' data, which will be introduced shortly after.

In line~\ref{line:detect_merge}, when the global server receives the privatized clustered information from all servers, it performs a merge operation to merge clusters from different servers whose estimated clustered preference vectors are close into a global clusters according to the following inequality.
\begin{equation}\label{eq:merge}
    \norm{\hat{\btheta}^{\ell _1}_{s, j_1}-\hat{\btheta}^{\ell_2}_{s, j_2}} < \alpha_2 (F(\tilde{T}^{\ell_1}_{s, j_1})+F(\tilde{T}^{\ell_2}_{s, j_2})),
\end{equation}
where $\hat{\btheta}^{\ell}_{s, j}= (\tilde{\bV}^\ell_{s, j})^{-1}\tilde{\bb}^\ell_{s,j}$ and $F(x)=\sqrt{\frac{1+\ln (1+x)}{(1+x)}}$.
$P_s$ denotes the set of $m_s$ global clusters.
Note that the local clusters in the same global cluster indexed by $k \in [m_s]$ will communicate and share protected clustered information with each other in an asynchronous manner.
At the beginning of phase $s$, if the new global cluster structure $P_s$ is different from $P_{s-1}$ at phase $s-1$, we will renew the shared global information $(\bS^{g}_{t,k}$, $\bu^{g}_{t,k}$, $T^{g}_{t,k})$ for $k \in [m_s]$.
For local servers $(\ell, j)$ in the same global cluster $P_{s,k}$, the local synchronized information $(\bS^{\ell}_{t,j}, \bu^{\ell}_{t,j}, T^{\ell}_{t,j})$, the upload buffers $(\Delta \bS^{\ell}_{t,j}, \Delta \bu^{\ell}_{t,j}, \Delta T^{\ell}_{t,j})$ and download buffers $(\Delta \bS^{-\ell}_{t,j}, \Delta \bu^{-\ell}_{t,j}, \Delta T^{-\ell}_{t,j})$ are renewed.
We also generate a new tree-based privatizer PVT$(\ell, j)$ for each cluster $j$ at server $\ell$, which will be introduced later on.

\begin{algorithm}[t]
	\resizebox{0.95\columnwidth}{!}{
\begin{minipage}{\columnwidth}
\begin{algorithmic}[1]
\caption{Phase-based FCLUB with CDP}\label{alg:main}
\STATE \textbf{Input:} Failure probability $\alpha$, deletion parameter $\alpha_1 > 0$, merge parameter $\alpha_2 > 0$, privacy parameters $\varepsilon, \delta$.
% \STATE \textbf{Notations}:Let $t_{u}$ represents the communication times, $m=\log(t_{u}) +1$, $\sigma = \frac{4\times(L_{1}^{2}+1)\sqrt{m\log(16/\delta_{1}^2)}}{\epsilon_{1}}$, $\rho_{min} = \sqrt{32}m(L_{1}^{2}+1)\log(4/\delta_{1})(4\sqrt{d}+2\ln{\frac{2t_{u}}{\alpha}})/\epsilon_{1}$, $\rho_{max}=3\rho_{min}$, $\Upsilon=\sqrt{m(L_{1}^{2}+1)(\sqrt{d}+2\log(2t_{u}\alpha))/\sqrt{2\epsilon_{1}}} $, $\beta(t) =\sigma\sqrt{2\log(2/\alpha)+d\log(3+\frac{t_{u}L_{1}^{2}}{d\rho_{min}})}+LS\sqrt{\rho_{max}}+L\Upsilon $
% \STATE \textbf{Setting}: $L$ Local servers $\{1, ..., L\}$, $m$ true clusters $\cC=\{C_1, ..., C_m\}$, $n$ users $\cU=\{u_1, ..., u_n\}$ and the global server $\cG$. 
% Each local server $l$ contains $m_l$ clusters $\cC^\ell=\{C^\ell_1, ..., C^\ell_{m_l}\}$, $n_l$ users $\cU^\ell=\{\bu^\ell_1, ..., \bu^\ell_{n_l}\} \subseteq \cU$.
\STATE User initialization: For $i \in [n]$, $ \bV_{0,i}=\boldsymbol{0}_{d \times d}, \bb_{0,i}=\boldsymbol{0}_{d \times 1}, \bT_{0,i}=0$.
\STATE Local server initialization: For $\ell \in [L]$, set graph $\cG^{\ell}_0=(\cU^{\ell}, \cE^{\ell}_0)$, local information, upload buffers, download buffers: $(\bS^{\ell}_{0,1}, \bu^{\ell}_{0,1}, T^{\ell}_{0,1})=(\Delta \bS^{\ell}_{0,1},\Delta \bu^{\ell}_{0,1},\Delta T^{\ell}_{0,1})=(\Delta \bS^{-\ell}_{0,1},\Delta \bu^{-\ell}_{0,1},\Delta T^{-\ell}_{0, 1})=(\boldsymbol{0}_{d \times d},\boldsymbol{0}_{d \times 1},0)$, perturbations $(\bar{\bH}^{\ell}_{t,1},\bar{\bh}^{\ell}_{t,1})=(\bH^{\ell}_{t,1},\bh^{\ell}_{t,1})=(\boldsymbol{0}_{d \times d},\boldsymbol{0}_{d \times 1})$.
\STATE Global server initialization: create one global cluster $P_0=\{\{(1,1),..., (L, 1)\}\}$ and set the global information for $V^g_{0,1}=\boldsymbol{0}_{d \times d}$, the $\bu^g_{0,1}=\boldsymbol{0}_{d \times 1}, T^g_{0,1}=0$.
\FOR{$s=1, 2, ..., $}
\STATE Detect and adjust clusters (\Cref{alg:cluster_detect}).
\FOR{$\tau=1, ..., 2^s$}
\STATE Compute the total time step $t=2^s-2 + \tau$.
\STATE Advance all parameter, e.g., $\bS^{\ell}_{t,j}=\bS^{\ell}_{t-1,j}$.
\STATE User $i_t$ at local server $l_t$ arrives and $l_t$ gets the local cluster $j_t$ that $i_t$ belongs to based on $\cG^\ell_{s}$.
% \STATE Compute $V_{t,j_t}=\bS^{l_t}_{t, j_t}+\Delta \bS^{l_t}_{t, j_t} , b_{t, j_t}=\bu^{l_t}_{t, j_t} + \Delta \bu^{l_t}_{t, j_t}$ and $N_{t, j_t}=T^{l_t}_{t, j_t} + \Delta T^{l_t}_{t, j_t}$.
\STATE Compute local $\beta^{l_t}_{t,j_t}$ according to \Cref{lemma:cdp_beta}.

\STATE Local server $l_t$ receives feasible context set $\bD_t$ and recommends item $\bx_t=\argmax_{\bx \in \bD_t} \bx^{\top}(\bS^{l_t}_{t, j_t})^{-1}\bu^{l_t}_{t, j_t} + \beta^{l_t}_{t,j_t} \norm{\bx}_{(\bS^{l_t}_{t, j_t})^{-1}}$.\label{line:recommend}
\STATE User $i_t$ receives feedback $y_t$, update the user $i_t$'s information: $(\bT_{t,i_t}, \bV_{t,i_t}, \bb_{t, i_t}) \mathrel{{+}{=}} (1,\bx_{t} \bx_{t} ^\top, y_t \bx_{t})$
% \STATE Update upload buffer $\bT_{t,i_t} = \bT^{l_t}_{t,i_t}+1, \bV_{t,i_t}= \bV^{l_t}_{t,i_t}+\bx_t \bx_t ^\top +B_t, \bb^{l_t}_{t, i_t}=\bb_{t, i_t}+y_t \bx_t + \xi_t.$
and others unchanged.
% : $\bT_{t,i} = \bT_{t,i}, \bV_{t,i}= \bV_{t,i}, \bb_{t, i}=\bb_{t, i}$ for all $(l,i) \neq (l_t,i_t)$.
\STATE Check upload event (\Cref{alg:check_upload}).
\STATE Check download event (\Cref{alg:check_download}).
\ENDFOR
\ENDFOR
\end{algorithmic}
\end{minipage}}
\end{algorithm}

\textbf{Asynchronous Communication Protocol.}
In this work, we design a novel asynchronous communication protocol to incorporate the randomly arriving users.
To reduce the communication cost, our high-level idea is to use the delayed communication, where the feedback are temporarily stored in buffers and only if the stored information exceeds a threshold, the upload/download events are triggered.
Such a threshold will ensure that the local information will not diverge too far from the global information, which in turn will not diverge too far from the scenario when information are fully synchronized.   
Also note that our communication is conducted in the asynchronous manner at the local cluster level.
In other words, all local clusters indexed by $(\ell, j) \in P_{s,k}$ will establish connection with each other within the global cluster $k$.
For each local cluster $(\ell, j)$, it stores a local copy of the sufficient statistics $(\bS^{\ell}_{t,j}, \bu^{\ell}_{t,j}, T^{\ell}_{t,j})$ and a upload buffer $(\Delta \bS^{\ell}_{t,j}, \Delta \bu^{\ell}_{t,j}, \Delta T^{\ell}_{t,j})$.
For the global server, it prepares for each local cluster $(\ell, j)$ a download buffer $(\Delta \bS^{-\ell}_{t,j}, \Delta \bu^{-\ell}_{t,j}, \Delta T^{-\ell}_{t,j})$, which are used to send other local servers' information to the local cluster.
It also maintains the global statistics $(\bS^{g}_{t,k}, \bu^{g}_{t,k}, T^{g}_{t,k})$ to save the data uploaded from local clusters in global cluster $k$.

Our proposed communication framework consists of two components: the upload protocol (\Cref{alg:check_upload}) and the download protocol (\Cref{alg:check_download}).
For the upload protocol, at each time step $t$, user $i_t$ visits the server $l_t$ and receives recommended item $\bx_t$.
After the user $i_t$ interacts with the environment and observes feedback $(\bx_{t}, y_t)$, the local server updates the upload buffers in line~\ref{line:upload_update_obs} and checks the following condition to decide whether to upload the upload buffer:
\begin{equation}\label{eq:upload_condition}
\text{det}(\bS^{l_t}_{t, j_t}+\Delta \bS^{l_t}_{t, j_t})/{\text{det}(\bS^{l_t}_{t, j_t})}\ge U,
\end{equation}
where $\bH^{l_t}_{t,j_t}$ and $\bar{\bH}^{l_t}_{t,j_t}$ are tentative and current perturbation for privacy protection, respectively. 
If the condition is satisfied, the local server sends $(\Delta \bS^{l_t}_{t, j_t}, \bu^{l_t}_{t, j_t}, T^{l_t}_{t, j_t})$ to the global cluster $k_t$.
The global server then merges the uploaded information into the global information in line~\ref{line:upload_merge} and also sends it to download buffers for other local clusters $(\ell,j) \neq (l_t, j_t) \in P_{s,k_t}$ in line~\ref{line:upload_renew_download}.
For local cluster $(l_t,j_t)$ itself, the local server updates the local statistics and initializes the upload buffer using the newly generated perturbation in lines~\ref{line:upload_update_local} to \ref{line:upload_clear}.
For the download protocol, at each time step $t$, the global server will check the deviation between global statistics and the local statistics via following condition:
\begin{equation}\label{eq:download_condition}
    \text{det}(\bS^g_{t,k_t})/{\text{det}(\bS^{\ell}_{t, j})}\ge D
\end{equation}
independently for local clusters $(\ell, j)$ in global cluster $P_{s,k_t}$.
If any cluster $(\ell,j)$ satisfies such condition, the global server sends the information from other clusters to $(\ell,j)$, which is used to update $(\ell,j)$'s local statistics.
Finally, the global server cleans the download buffer.

\textbf{Tree-Based Privacy Protocol.}
To ensure the uploaded information are privatized, we adopt the tree-based privatizer to generate random perturbations $\bH^{\ell}_{t,j}$ and $\bh^{\ell}_{t,j}$ whenever an upload event happens.
Note that the privatier subroutine is at the local cluster level and a new privatizer is created if the cluster structure changes at the start of any phase $s$.

Let $\bx_1, ..., \bx_T$ be a (matrix-valued) sequence of length $T$, and $s_i = \sum_{t=1}^i \bx_i$ be the partial sum of the first $i$ elements that will be realised privately.
Generally speaking, the tree-based mechanism~\cite{dwork2006calibrating} maintains a binary tree $\cT$ of depth $1+ \lceil \log T \rceil$, where the leaf nodes contain the elements $\bx_i$ and the parent node maintains the sum of its children. For each node with value $n_i$, the tree-base mechanism protects privacy by adding noise $h_i$ to each node and release $n_i + h_i$ if queried.
The key advantage is that such a tree only accesses $\nu=O(\log T)$ nodes to compute and release the partial sum $s_i$, which means the perturbation is at most $O(\nu)$ instead of $O(T)$.

Following this general idea, we implement the tree-based privatizer $(\ell, j)$ that satisfies the requirements of CDP.
Recall that we only need to protect the information uploaded to the global server, it suffices to maintain a tree $\cT^{\ell}_{j}$ of depth $\nu=O(1+\lceil \log t_c \rceil)$ for the upload event, where $t_c$ is the total number of uploads.
To make the partial sums private, we insert a random noise matrix to each node in $\cT^{\ell}_j$, similar to that of \citet{shariff2018differentially} and \cite{dubey2020differentially},.
Specifically, we sample a random matrix $\bar{\bN} \in \R^{(d+1)\times (d+1)}$ where each entry $\bar{N}_{p,q}$ is drawn from i.i.d. Gaussian distribution $\cN(0, \sigma_{\text{noise}})$ and symmetrize it to get $N=(\bar{\bN}^{\top}+\bar{\bN})/\sqrt{2}$.
It follows that in order to ensure the whole tree is $(\varepsilon, \delta)$-DP, each node should preserve $(\varepsilon/\sqrt{8\nu\log (2/\delta}), \delta/2)$-DP.
In other words, it suffices to set the variance $\sigma_{\text{noise}}=64\nu \log (2/\delta)^2/\epsilon^2$ for each tree node.
Note that at each upload round $t$, the total noise added to the partial sum is the summation of at most $\nu$ random matrices with size $(d+1)\times (d+1)$, where the top-left $(d \times d)$-submatrix forms $\bH^{\ell}_{t,j}$ and the first $d$ elements from the right-most $(d+1)\times 1$ vector forms $\bh^{\ell}_{t,j}$.
By concentration of random matrices~\cite{tao2011topics}, we have with probability at least $(1-\frac{\alpha}{mL})$, the operator norm of $\bH^{\ell}_{t,j}$ is
\begin{equation}
    \norm{\bH^{\ell}_{t,j}}_{op} \le \rho \triangleq 8\sqrt{2}\nu \log (4/\delta) (4\sqrt{d} + 2 \log (2mL/\alpha))/\varepsilon.
\end{equation}
for any $\ell \in [L], j \in [m], t \in [T]$.

% With a little abuse of notation, suppose the local cluster $(\ell, j)$ receives observation $(\bx_t, y_t)_{t \in [T]}$ and the upload happens in $(t_1, ..., t_{t_c})$ rounds.
% Consider two consecutive uploads occurs at time $t_{i-1}$ and $t_{i}$, the local cluster will insert $\sum_{\tau=t_{i-1}}^{t_{i}}$ 
% , since each partial sum $s_i$ can be uniquely represented by the summation of at most $\nu$ nodes in $\cT$.

\textbf{Recommendation Procedure.}
At each time step $t$, the recommended item $\bx_t$ for user $i_t$ is selected as follows.
When the current cluster is correct (which is guaranteed after $O(\log T)$ rounds and to be proved later), the estimated $\hat{\btheta}_{t}=(\bS^{l_t}_{t, j_t})^{-1}\bu^{l_t}_{t, j_t}$ is computed using the local information $\bS^{l_t}_{t, j_t}$ and $\bu^{l_t}_{t, j_t}$.
Since by \Cref{lemma:cdp_beta}, $\norm{\btheta_{i_t}-\hat{\btheta}_{t}}_2 \le  \beta^{l_t}_{t,j_t}$, the confidence radius is $ \beta^{l_t}_{t,j_t}\norm{\bx}_{(\bS^{l_t}_{t, j_t})^{-1}}$, which characterizes the exploration bonus for item $\bx \in \bD_t$.
Then the local server will recommend the item $\bx_t \in \bD_t$ that maximizes the $\bx^{\top}\hat{theta}_t$ plus the above exploration bonus.
Finally, the user will receive feedback $y_t$ and the system updates corresponding statistics for better decision in future rounds.
\begin{restatable}{lemma}{BetaCi}\label{lemma:cdp_beta}
Under the setting of FCLUB and fix a local cluster $j$ located at the server $\ell$ which shares the information with $L'\le L$ clusters (including itself), let the true preference vector be $\btheta^*$ and the true cluster be $j^*$, let $\hat{\btheta}^{\ell}_{t,j}=(\bS^{\ell}_{t,j})^{-1} \bu^{\ell}_{t,j}$.
When all (global) clusters are correctly identified and partitioned, it holds with probability at least $1-2\alpha$,
\begin{equation}
    \norm{\btheta^*-\hat{\btheta}^{\ell}_{t,j}}_{\bS^{\ell}_{t,j}} \le \beta^{\ell}_{t,j},
\end{equation}
where $\beta^{\ell}_{t,j} \triangleq \beta^{\ell}_j(T^{\ell}_{t,j}, L, \alpha/(mL))= \sigma_{0}\sqrt{2\log(\frac{mL}{\alpha})+d\log(\frac{\rho_{\max}}{\rho_{\min}}+\frac{T^{\ell}_{t,j}}{dL'\rho_{\min}})} + \sqrt{L'\rho_{\max}} +\sqrt{L'}\kappa$.
\end{restatable}

\begin{algorithm}[t]
	\resizebox{0.95\columnwidth}{!}{
\begin{minipage}{\columnwidth}
\begin{algorithmic}[1]
\caption{Phase-based Cluster Detection and Adjustment}\label{alg:cluster_detect}
\STATE $t=2^s-1$.
\FOR{$\ell \in [L]$}
\STATE Set $\cG^{\ell}_{s}$ by deleting any edge $(i_1,i_2) \in \cG^{\ell}_{s-1}$ if \Cref{eq:edge_deletion} holds. \label{line:detect_delete}
\STATE For $\ell \in [L], j \in C(\cG^{\ell}_{s}) $, generate new perturbation $\bH^{\ell}_{t,j}, \bh^{\ell}_{t,j}$ using $\text{PVT}({\ell, j})$ in \Cref{alg:private} and set historical $\bar{\bH}^{\ell}_{t,j}=\bH^{\ell}_{t,j}, \bh^{\ell}_{t,j}=\bar{\bh}^{\ell}_{t,j}$. 
\STATE For $\ell \in [L]$, upload the local clustered information $I_{s,{\ell}}=(C^{\ell}_{s,j}, \tilde{\bV}^{\ell}_{s, j}, \tilde{\bb}^{\ell}_{s, j}, \tilde{T}^{\ell}_{s,j})_{j \in C(\cG^{\ell}_{s})}$ to the global server, where $(\tilde{\bV}^{\ell}_{s,j}, \tilde{\bb}^{\ell}_{s,j}, \tilde{T}^{\ell}_{s,j})= (2\rho \bI + \bH^{\ell}_{t,j},\bh^{\ell}_{t,j}, 0) +\sum_{i \in C^{\ell}_{s,j}}(\bV_{t,i},\bb_{t,i},T_{t,i})$.\label{line:detect_cluster}
\ENDFOR
\STATE The global server does global merge based on $I_s$ and get $m_s$ global clusters $P_s=\{P_{s, 1}, ..., P_{s, m_s}\}$, where the two local clusters $C^{\ell_1}_{t,j_1}$, $C^{\ell_2}_{t,j_2}$ (with $\ell_1 \neq \ell_2$) are merged together in $P_{s,k}$ if \Cref{eq:merge} holds.\label{line:detect_merge}
\IF{$s=0$ or $P_s\neq P_{s-1}$}
\STATE //Renew the cluster information.
\FOR{$k \in [m_s]$}
\STATE Set global gram matrix $(\bS^g_{t, k}, \bu^g_{t, k},T^g_{t, k})=\sum_{(\ell,j) \in P_{s,k}}(\tilde{\bV}^\ell_{s,j},\tilde{\bb}^\ell_{s,j},{\tilde{T}^\ell_{s,j}})$.
    \FOR{$(\ell,j) \in P_{s, k}$}
\STATE Set $(\bS^\ell_{t, j},b^\ell_{t, j},T^{\ell}_{t, j})=(\bS^g_{t,k}, \bb^g_{t,k}, T^g_{t,k})$.
\STATE Create new perturbation $\bH^{\ell}_{t,j}, \bh^{\ell}_{t,j}$ using $\text{PVT}({\ell, j})$ in \Cref{alg:private}.
\STATE Set new $(\Delta \bS^\ell_{t,j}, \Delta \bu^{\ell}_{t,j}, \Delta T^{\ell}_{t,j})=(3\rho \bI + \bH^\ell_{t,j}-\bar{\bH}^\ell_{t,j}, \bh^\ell_{t,j}-\bar{\bh}^\ell_{t,j},0)$
\STATE Set new $(\Delta \bS^{-\ell}_{t,j}, \Delta \bu^{-\ell}_{t,j}, \Delta T^{-\ell}_{t,j})=(\boldsymbol{0},\boldsymbol{0},0)$.
    \ENDFOR
\ENDFOR
\ENDIF
\end{algorithmic}
\end{minipage}}
\end{algorithm}

\begin{algorithm}[t]
	\resizebox{0.95\columnwidth}{!}{
\begin{minipage}{\columnwidth}
\begin{algorithmic}[1]
\caption{Check Upload Event}\label{alg:check_upload}
\STATE Update upload buffer $(\Delta \bS^{l_t}_{t,j_t}, \Delta \bu^{l_t}_{t, j_t}, \Delta \bT^{l_t}_{t,j_t}) \mathrel{{+}{=}} (\bx_{t} \bx_{t} ^\top,y_t \bx_{t},1)$.\label{line:upload_update_obs}
\IF{$\text{det}(\bS^{l_t}_{t, j_t}+\Delta \bS^{l_t}_{t, j_t})/\text{det}(\bS^{l_t}_{t, j_t})\ge U$}
\STATE The global cluster finds $k_t$ so that $(l_t, j_t) \in P_{s, k_t}$.
% \STATE Local server sets $\Delta \bS^{l_t}_{t, j_t}=\Delta \bS^{l_t}_{t, j_t}+\bH^{l_t}_{t,j_t}-\bar{\bH}^{l_t}_{t,j_t}$ and $\Delta \bu^{l_t}_{t, j_t}= \Delta \bu^{l_t}_{t, j_t}+ \bh^{l_t}_{t,j_t}-\bar{\bh}^{l_t}_{t,j_t}$.
\STATE Update global information $(\bS^g_{t, k_t}, \bu^g_{t, k_t}, T^g_{t, k_t})\mathrel{{+}{=}}(\Delta \bS^{l_t}_{t, j_t}, \Delta \bu^{l_t}_{t, j_t},\Delta T^{l_t}_{t, j_t})$.\label{line:upload_merge}
\FOR{$(\ell,j) \neq (l_t, j_t) \in P_{s, k_t}$}
\STATE Global server updates other servers' download buffer $(\Delta \bS^{-\ell}_{t,j},\Delta \bu^{-\ell}_{t,j}, \Delta T^{-\ell}_{t,j}) \mathrel{{+}{=}}( \Delta \bS^{l_t}_{t, j_t}, \Delta \bu^{l_t}_{t, j_t}, \Delta T^{l_t}_{t, j_t})$.\label{line:upload_renew_download}
\ENDFOR
\STATE Local server $l_t$ updates the local statistics: $(\bS^{l_t}_{t, j_t}, \bu^{l_t}_{t, j_t}, T^{l_t}_{t, j_t}) \mathrel{{+}{=}} (\Delta \bS^{l_t}_{t, j_t}, \Delta \bu^{l_t}_{t, j_t}, \Delta T^{l_t}_{t, j_t})$.\label{line:upload_update_local}
\STATE Local server $l_t$ sets $(\bar{\bH}^{l_t}_{t,j_t}, \bar{\bh}^{l_t}_{t,j_t})= (\bH^{l_t}_{t,j_t}, \bh^{l_t}_{t,j})$ and creates new perturbation $\bH^{\ell}_{t,j_t}, \bh^{\ell}_{t,j_t}$ using the tree-based privatizer $\text{PVT}({l_t, j_t})$.\label{line:upload_perterb}
\STATE Local server $l_t$ initializes the upload buffer using the new perturbation: $(\Delta \bS^{l_t}_{t, j_t}, \Delta \bu^{l_t}_{t, j_t}, \Delta T^{l_t}_{t, j_t})=(3\rho \bI+\bH^{l_t}_{t,j}-\bH^{l_t}_{t,j_t}, \bh^{l_t}_{t,j_t}-\bar{\bh}^{l_t}_{t,j_t},0)$.\label{line:upload_clear}
\ENDIF
\end{algorithmic}
\end{minipage}}
\end{algorithm}

\begin{algorithm}[t]
	\resizebox{0.95\columnwidth}{!}{
\begin{minipage}{\columnwidth}
\begin{algorithmic}[1]
\caption{Check Download Event}\label{alg:check_download}
\FOR{$(l,j) \in P_{s,k_t}$}
\IF{$\text{det}(\bS^g_{t,k_t})/{\text{det}(\bS^{\ell}_{t, j})}\ge D$}
\STATE Local server receives $(\bS^{\ell}_{t,j}, \bu^{\ell}_{t,j}, T^{\ell}_{t,j}) \mathrel{{+}{=}} (\Delta \bS^{-\ell}_{t,j}, \Delta \bu^{-\ell}_{t,j}, \Delta T^{-\ell}_{t,j})$.
\STATE Global server cleans the download buffer: $(\Delta \bS^{-\ell}_{t,j}, \Delta \bu^{-\ell}_{t,j}, \Delta T^{-\ell}_{t,j})=(\boldsymbol{0},\boldsymbol{0},0)$.
\ENDIF
\ENDFOR
\end{algorithmic}
\end{minipage}}
\end{algorithm}

\begin{algorithm}[t]
	\resizebox{0.95\columnwidth}{!}{
\begin{minipage}{\columnwidth}
\begin{algorithmic}[1]
\caption{Privatizer PVT($\ell,j$) for cluster $j$ at server $\ell$} \label{alg:private}
\STATE \textbf{Input:} Privacy budget $\varepsilon, \delta$, number of uploads $t_c$.
% , number of clusters $L'$ to share the information with cluster $j$ at server $\ell$.
\STATE Create a binary tree $\cT$ of depth $\nu=\lceil \log (t_c+1) \rceil+1$.
\STATE For each node, we generate a perturbation matrix matrix $N \in \R^{(d+1)\times (d+1)}$, where $\bN=({\bar{\bN}+\bar{\bN}^{\top}})/\sqrt{2}$ and $\bar{\bN} \in \R^{(d+1)\times (d+1)}$ with $\bar{N}_{p,q}\sim \cN(0, 64\nu\frac{\log (2/\delta)^2}{\varepsilon^2})$.
\STATE Calculate a queue of $\cQ=(\bH_i, \bh_i)_{i=1, ..., t_c+1}$ for partial sums $s_0, ..., s_{t_c}$.
\STATE Sequentially pop one pair of $\cQ$ if PVT($\ell,j$) is called.
\end{algorithmic}
\end{minipage}}
\end{algorithm}

\section{Results}\label{sec:results}
Recall that perturbations $\bar{\bH}^{\ell}_{t,j}, \bar{\bH}^{\ell}_{t,j}$ are designed to satisfy the $(\varepsilon, \delta, L, m)$-CDP requirement.
In particular, the privacy budget $(\varepsilon,\delta)$ affects the regret and communication bounds via the following quantities $(\rho_{\max}, \rho_{\min}, \kappa)$, which can be treated as spectral bounds for $\bar{\bH}^{\ell}_{t,j}, \bar{\bH}^{\ell}_{t,j}$.
Let $\tilde{\bH}^{\ell}_{t,j}=2\rho \bI + 3\rho  c^{\ell}_{j,t} \bI + \bar{\bH}^{\ell}_{t,j}$, where $c^{\ell}_{j,t}$ is the number of uploads for local server $\ell$ and cluster $j$.
\begin{definition}[Approximately-accurate $\rho_{\min}, \rho_{\max}$ and $\kappa$]\label{def:accurate}
The bounds $0<\rho_{t, \min} \le \rho_{t,\max}$ and $\kappa >0$ are $(\alpha/(mL))$-accurate for $(\bar{\bH}^{\ell}_{t,j})$ for any $\ell \in [L], j \in [m]$ and $t \in [T]$:
\begin{equation}\label{eq:alpha_accurate}
    \norm{\tilde{\bH}^{\ell}_{t,j}}_{\text{op}} \le \rho_{\max}, \norm{(\tilde{\bH}^{\ell}_{t,j})^{-1}}_{\text{op}} \le \frac{1}{\rho_{\min}}, \norm{\bar{\bh}^{\ell}_{t,j}}_{(\tilde{\bH}^{\ell}_{t,j})^{-1}} \le \kappa
\end{equation}
\end{definition}
with probability at least $(1-\frac{\alpha}{mL})$. 

As will be shown later, our communication protocol ensures $ c^{\ell}_{j,t} \in [0, d\log T/\log (\min\{U,D\})]$, so $\rho_{\min} = \rho$,  $\rho_{\max} = 3\rho + 3\rho d\log T/\log (\min\{U,D\})$ , and $\kappa=\norm{\bar{\bh}^{\ell}_{t,j}}/\sqrt{\rho}$, where $\rho \triangleq 8\sqrt{2}\nu \log (4/\delta) (4\sqrt{d} + 2 \log (2mL/\alpha))/\varepsilon$ is given by our privatizer.

In the following, we will give general regret and communication bounds using $(\rho_{\max}, \rho_{\min}, \kappa)$ and replace them with their exact values.

\subsection{Regret Bound}
We give the following theorem as our main result for the regret bound.
\begin{restatable}{theorem}{mainRegret}\label{thm:mainRegret}
Suppose the cluster structure over the users and items satisfy the assumptions in \Cref{sec: problem setting} with gap parameter $\gamma>0$ and item regularity parameter $1\ge \lambda_x>0$. If the privatizer produces random perturbation that are $(1/(8mLT))$-accurate as in \Cref{def:accurate}, with probability at least $1-1/T$, the regret is upper bounded by
\begin{align}
    % &R(T) \le 2T_0 + \sum_{j=1}^m R_j(T_j)\\
%     &\le  2T_0 + \sum_{j=1}^m\Big(\sigma_{0}\sqrt{2\log(8mLT)+d\log(\frac{\rho_{\max}}{\rho_{\min}}+\frac{T_{j}}{d\rho_{\min}})}+\\
%     &\sqrt{L\rho_{\max}}+\kappa\sqrt{L}\Big)\Big(\sqrt{d\log(\frac{\rho_{\max}}{\rho_{\min}}+\frac{T_{j}}{d\rho_{\min}})}\Big)\Gamma\sqrt{T_j} \\
% &\le  O\Big(2T_0 + \sum_{j=1}^m\sqrt{L\frac{d\log T}{\log \min\{U,D\}} \frac{\log\log T \log (1/\delta) (\sqrt{d}+\log T)}{\varepsilon}}\\
% &\Gamma\sqrt{d\log T }\sqrt{T_j}\Big)\\
R(T)&\le \tilde{O}\Big(n(\frac{\log T}{\lambda_x^2}+\frac{\sigma_0^2d\log T}{\lambda_x \gamma^2} + \frac{\log (1/\delta)\log T }{\lambda_x \varepsilon \gamma^2})\notag\\
&+ dL\sqrt{mT\frac{\log(1/\delta)}{\varepsilon}}\log^{1.5} T \Big)
\end{align}
\end{restatable}
We will give the proof sketch for the above \cref{thm:mainRegret}.
\begin{proof}
Our proof mainly consists of two parts.
The first part bounds the number of exploration rounds $2T_0$ after which the overall user clusters are correctly detected at the global server.
The second part is to bound the regret for the asynchronous contextual linear bandits after the clusters are partitioned correctly.

Different from standard online clustering bandits, the key technical challenge is to take care of the additional random Gaussian noise produced by the privater, which perturbs the true observation that is needed for global cluster detection and the regret analysis for contextual linear bandits.
Moreover, such perturbed observation are also lagged behind the instant observation, since FCLUB-CDP adopts the "delayed" asynchronous communication where upload and download are triggered occasionally. 
This makes standard contextual bandit analysis no longer works and requires new proof techniques to handle the gap between instant observation and the lagged (and perturbed) observation.

For the first cluster detection part, by the assumption of item regularity, we prove that after $t \ge O(n(\frac{\log T}{\lambda_x^2}+\frac{d\sigma_0^2\log T}{\lambda_x \gamma^2}))$ rounds, the local estimates are accurate enough so that the local clusters are correctly identified, similar to that of~\cite{li2018online}.
Specifically, the 2-norm distance between local estimate $\hat{\theta}_{t,i}$ and the truth $\theta_i$ for any user $i$ is less than $\gamma/4$.
Thus the local clusters are split correctly for all local servers. 
Now for the global cluster detection, the global server receives the aggregated observation from correctly partitioned local clusters, in which random Gaussian noises are added.
Based on spectra property of Gaussian noise matrices (\cref{def:accurate}), the global server will spend additional $O(\frac{n\log (1/\delta)\log T }{\lambda_x \varepsilon \gamma^2})$ rounds so that the perturbed estimate $\hat{\theta}^{l}_{s,j}$ are accurate enough at the beginning of phase $s=\lceil \log_2 T_0\rceil$, where $T_0=O(n(\frac{\log T}{\lambda_x^2}+\frac{d\sigma_0^2\log T}{\lambda_x \gamma^2} + \frac{\log (1/\delta)\log T }{\lambda_x \varepsilon \gamma^2}))$.
Therefore, after $t > 2T_0$, the overall user clusters are partitioned correctly.

For the regret after $2T_0$, we use the delayed update technique from~\citep[Section 5.1]{abbasi2011improved}, which only recomputes the confidence radius only $O(\log T)$ times and hence saves computation.
The same strategy can also be applied for the delayed communication.
The key analysis relies on using the upload and download condition in~\cref{eq:upload_condition} and \cref{eq:download_condition}, so that the actually-used cluster confidence radius is at most $\Gamma$ times larger than that if all local servers upload their perturbed observations in a fully synchronized manner, where $\Gamma=\sqrt{D(1+(L-1)(U-1))+ U-1}$.
This will give a $\Gamma R_j(T_j)$ regret for the second part, where $R_j(T_j)$ is the private-version regret for the cluster $j$ if all observation are synchronized at each round. The full proof is put in \OnlyInFull{\Cref{sec:regret_proof}}\OnlyInShort{Appendix B}.
\end{proof}

\subsection{Communication Cost}
We give the following theorem to bound the total communication cost.
\begin{restatable}{theorem}{comCost}\label{thm:comCost}
Under the CDP setting, the total communication cost satisfies:
\begin{align}
    % C_{T} &\le mL\Big(\log T + \frac{d\log(\frac{\rho_{\max}}{\rho_{\min}}+\frac{T}{d\rho_{\min}})}{ \log(\min\{U,D\})} \\
    % &+\frac{d\log(\frac{\rho_{\max}}{\rho_{\min}}+\frac{2T_0}{d\rho_{\min}})\log(2T_0)}{ \log(\min\{U,D\})}\Big)\\
    C(T)&\le O(\frac{dmL\log T}{\log(\min\{U,D\}})
\end{align}
\end{restatable}
\begin{proof}
The total communication cost also has two parts: the upload at the beginning of each phase for global cluster detection and the asynchronous communication within each phase for information sharing.
For the first part, the algorithm has at most $\log T$ phases and at each phase, there are total $mL$ local clusters uploading the clustered information, hence the total communication cost is $O(mL\log T)$.
For the second part, recall that we adopt the delayed asynchronous communication protocol and the total number of uploads and downloads can be bounded by $O(dmL\log T)$.
% This implements the intuition that the more observation we have the more difficult it is to trigger the upload and download condition in ~\Cref{eq:upload_condition} and ~\Cref{eq:download_condition}.
See \OnlyInFull{\Cref{sec:communication_cost}}\OnlyInShort{Appendix C} for the detailed proofs.
\end{proof}

\subsection{Privacy Guarantee}
\begin{restatable}{theorem}{privacyCDP}\label{thm:privacyCDP}
\Cref{alg:main} preserves $(\varepsilon,\delta, L, m)$-CDP as defined in \Cref{def:cdp}.
\end{restatable}
\begin{proof}
The CDP condition is satisfied by assigning the right amount of Gaussian noise in each tree node of our tree-based privacy protocol in~\Cref{sec:algorithm}. See \OnlyInFull{\Cref{sec:privacy_proof}}\OnlyInShort{Appendix D} for details.
\end{proof}
\subsection{Discussion and Comparison}
\textbf{Discussion on the Regret Bounds.}
For the regret bound, our result has two terms: the regret before the clusters are correctly partitioned $n(\frac{\log T}{\lambda_x^2}+\frac{\sigma_0^2d\log T}{\lambda_x \gamma^2} + \frac{\log (1/\delta) \log T}{\lambda_x \varepsilon \gamma^2})$ and the regret after the clusters are correctly partition $O(dL\sqrt{mT\frac{\log(1/\delta)}{\varepsilon}}\log^{1.5} T)$. 
We will compare our results with several degenerate cases, given that we are the first work to study the federated clustering of bandits setting. 
For these cases, the additional CDP causes at most $O(\sqrt{\frac{\log(1/\delta)}{\varepsilon}})$ factor and asynchronous communication protocol causes at most $O(\sqrt{d\log T})$ factor in general.

First, when $m=1, L=1$, our setting degenerates to the linear bandits with DP where all users share the same underlying parameter.
Compared to \citet{shariff2018differentially} which gives a $O(\sqrt{d\frac{\log(1/\delta)}{\varepsilon}}\sqrt{T}\log^{1.5}T)$ regret with $0$ communication, our bound has a $O(\sqrt{d})$ additional factor (or more precisely $O(\sqrt{d\log \log T})$ factor) for the second term, which stems from the larger perturbation in order to protect total $O(d\log T)$ communication rounds.

Second, when $L=1$, our setting reduces to the online clustering bandits with DP, \citet{li2018online} gives a $O(n(\frac{\log T}{\lambda_x^2}+\frac{d\sigma_0^2\log T}{\lambda_x \gamma^2})+d\sqrt{mT}\log T)$ for the non-DP version.
Since CDP mechanism requires random perturbation, the clustering process suffers an additional $\frac{n\log T\log (1/\delta) }{\lambda_x \varepsilon \gamma^2}$ for the first term and the second regret term now has a new $\sqrt{\frac{\log(1/\delta)\log T}{\varepsilon}}$ leading factor due to the CDP requirements. 

Third, when $m=1$ and if we consider the special case when each local server only has one user and all users come in a round-robin manner, our setting reduces to the distributed linear bandits with DP.
\citet{dubey2020differentially} provides a synchronized algorithm that achieves $O(L\sqrt{dT\frac{\log(1/\delta)}{\varepsilon}}\log^{1.5} T)$, our second term has an additional $\sqrt{d}$ factor because of different communication protocol, which enables asynchronous communication at the cost of the larger $O(dL\log T)$ compared with $O(L\log T)$ communication rounds.
% and an additional $\Gamma=O(\sqrt{L})$ factor in the confidence radius.

Finally, there is a lower bound $\Omega(\sqrt{dmT})$,
if we consider the case where the clustering structure is known,
the communication and privacy budgets are unlimited and each cluster contains equal number of users. 
In this case, it is equivalent to learn $m$ independent linear bandits, each with expected rounds $T/m$ and according to \cite{dani2008stochastic}, the lower bound is $\Omega(\sum_{i \in [m]}\sqrt{dT/m})=\Omega(\sqrt{dmT})$.
In other cases, the regret lower bound will be greater and the lower bound $\Omega(\sqrt{dmT})$ still holds.
Our regret bound matches the lower bound up to a factor of $O(L\sqrt{d\frac{\log(1/\delta)}{\varepsilon}}\log^{1.5} T)$.
% in total, and up to a factor of $O(\sqrt{d}\log^{1.5} T)$ regarding $d, m, T$ terms. Now for the $L$ and $\delta,\varepsilon$ terms, which are related to communication and privacy, to the best of our knowledge, the lower bound results are still unclear and remain challenging open questions in the literature. 
% This is because the regret, the communication cost and the privacy are intertwined and for most cases, delicate trade-offs need to be made between different criteria.

\textbf{Discussion on the Communication Cost.}
Our communication cost also has two terms: the first $O(mL\log T)$ term for identifying clusters at the beginning of each phase and the leading $O(\frac{dmL\log T}{\log(\min\{U,D\}})$ term for our asynchronous communication protocol.
Compared with \citet{dubey2020differentially} when $m=1$ and users come at the round-robin manner, our communication has an additional $O(d)$ factor. Due to the specialty of the user arrival, the same paper can achieve communication cost independent of $T$ at the cost of $O(\log (LT))$ additional factor in the regret. Though our total communication cost can not be reduced below $O(mL\log T)$ due to the first term, it will be interesting to consider whether the similar trade-off works for our asynchronous protocol in the future work.

\begin{figure*}[t]
 \centering
 \begin{subfigure}[b]{0.495\textwidth}
  \centering
  \includegraphics[width=0.7\textwidth]{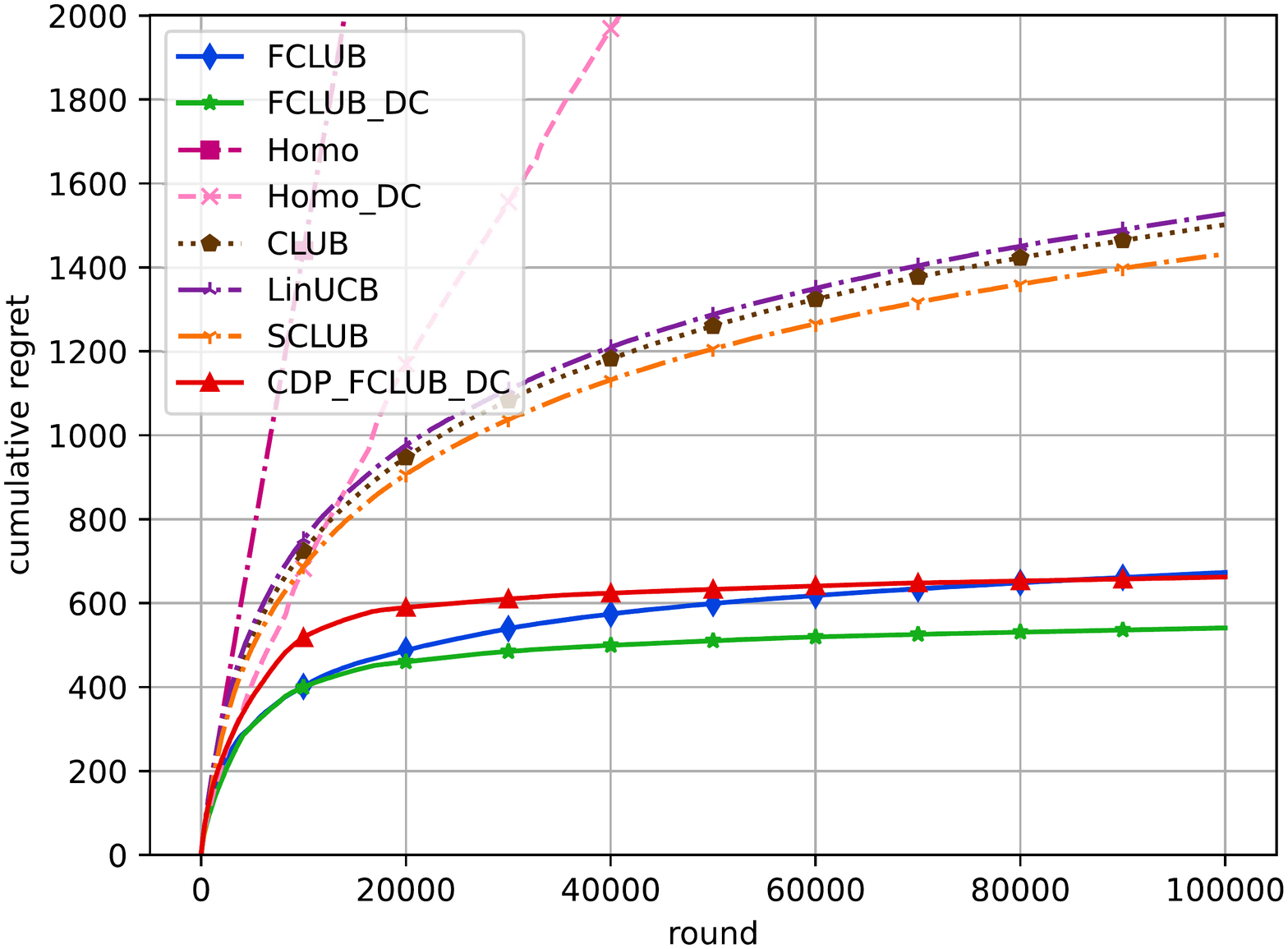}
  \caption{Comparison with Baselines on the Synthetic Dataset}
  \label{fig:synthetic}
 \end{subfigure}
 \hfill
 \begin{subfigure}[b]{0.495\textwidth}
  \centering
  \includegraphics[width=0.67\textwidth]{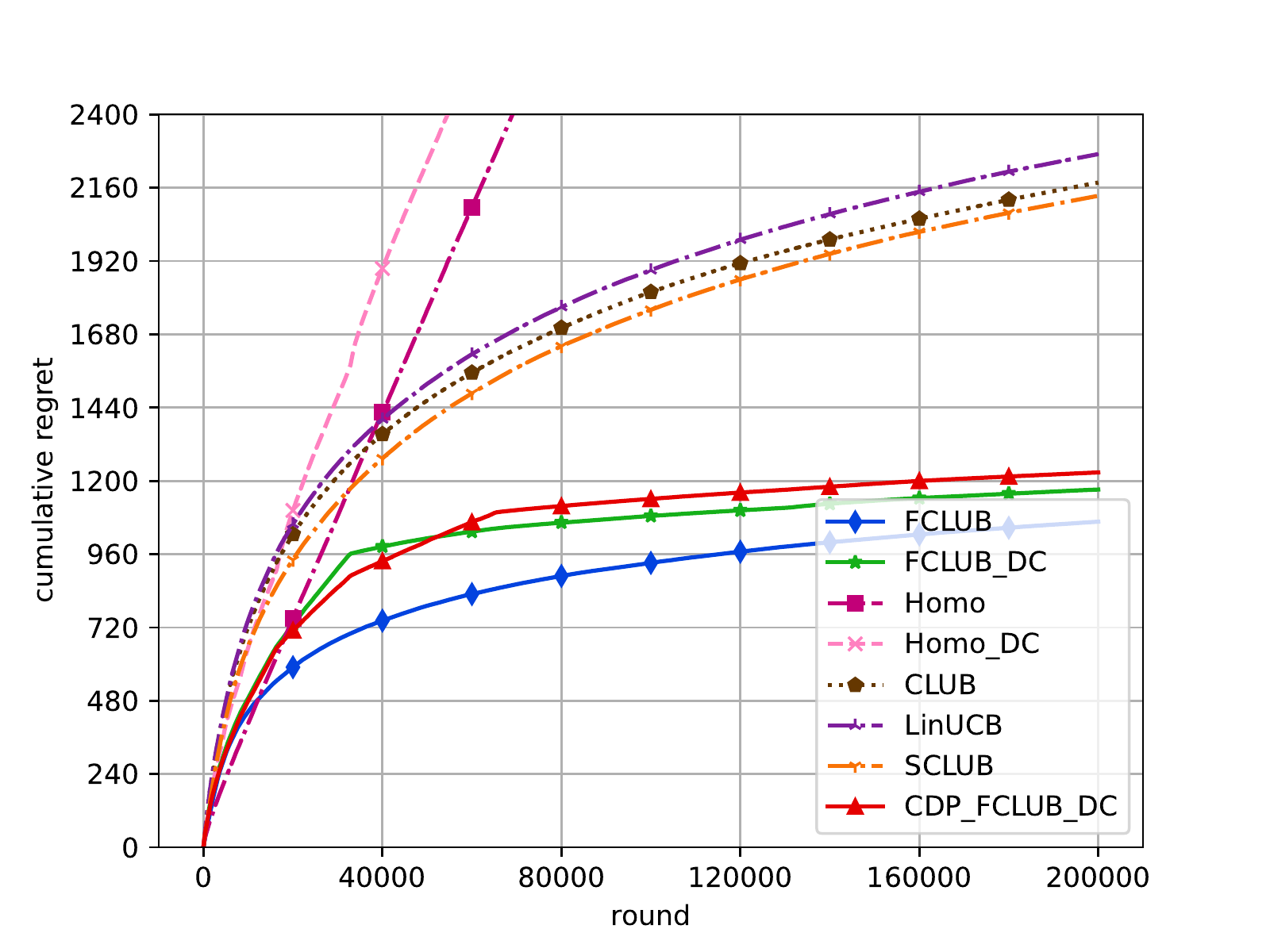}
  \caption{Comparison with Baselines on the MovieLens Dataset}
  \label{fig:movielens}
 \end{subfigure}
 \caption{Comparative Experiments on the Cumulative Regret with Baseline Algorithms}\label{fig:compare2}
\end{figure*}
\section{Experiments}\label{sec: experiments}
To validate our theoretical findings, we conduct experiments on a synthetic dataset and a real-world MovieLens dataset.
\Cref{alg:main} is denoted as CDP-FCLUB-DC and we also present its non-private version FCLUB-DC and its non-private, synchronized version FCLUB (by instantly uploading the observations).
% These two variants can better show the collaborative effects over different servers compared with non-private or synchronized algorithms.
The baselines include CLUB which uses a separate CLUB algorithm~\cite{gentile2014online} for each local server; SCLUB which uses a separate SCLUB algorithm~\cite{li2019improved} for each server; and LinUCB which uses a separate LinUCB algorithm~\cite{abbasi2011improved} for each user.
We also consider the synchronized and asynchronous version of \Cref{alg:main} (denoted as Homo and Homo-DC, respectively) by treating users are identical with the same preference vector. 
Note that all results are averaged over ten random seeds, and we provide mean results with one unit of standard derivation for each curve.
Due to the space limit, we provide the detailed experiment settings (including data generation and processing) in \OnlyInFull{\Cref{sec:exp_setting}}\OnlyInShort{Appendix E.1}, the parameter study \OnlyInFull{\Cref{sec:exp_pstudy}}\OnlyInShort{Appendix E.2}, the communication cost \OnlyInFull{\Cref{sec:exp_communication_cost}}\OnlyInShort{Appendix E.3} and running time results in \OnlyInFull{\Cref{sec:exp_running_time}}\OnlyInShort{Appendix E.4}, respectively.

\textbf{Synthetic Dataset.} We first conduct experiments on a synthetic dataset.
In Figure \ref{fig:synthetic}, we compare our algorithm CDP-FCLUB-DC with the baselines listed above.
The vertical axis indicates the cumulative regret and the horizontal axis indicates the round $t$.
In general, our algorithm CDP-FCLUB-DC's performance has a clear advantage over baseline SCLUB, CLUB, LinUCB, Homo and Homo-DC.
Since Homo-DC and Homo assumes users are in the same cluster, they mistakenly merge different clusters and suffer linear regrets, indicating the correctness of cluster detection is essential to have small regrets.
Compared with SCLUB and CLUB that only perform local clustering operations, we can verify the correctness of our algorithm's clustering operations at the global level, which successfully leverages the collaborative effects across different local servers.
As expected, CDP-FCLUB-DC performs a little worse than FCLUB and FCLUB-DC due to the delayed communication and cluster differential privacy requirements.

% In Homo-DC, we assume all users are in the same cluster, we can verify that our algorithm CDP-FCLUB-DC can complete cluster detection well by comparing it with Homo-DC. Besides, since CLUB and SCLUB only perform merge or split operations at the local cluster level, through comparison with SCLUB and CLUB we can verify the correctness of our algorithm's merge operation at the global level. In general, our algorithm CDP-FCLUB-DC's performance has a clear advantage over baseline SCLUB, CLUB, LinUCB, Homo and Homo-DC. CDP-FCLUB-DC performs worse than FCLUB and FCLUB-DC due to the delayed communication and cluster differential privacy(CDP).
\textbf{MovieLens Dataset.}
% In this section, we compare our algorithm CDP-FCLUB-DC with the baselines listed above on movie recommendations with Movielens dataset. We randomly draw $10^3$ movies and $10^3$ users as preparation for the experiment. Similar to \citet{li2018online}, we generate the weight vectors on the basis of the $10^3$ users and the movies they rated. We randomly select $40$ users from these $10^3$ users for the experiment. In this experiment, we assume that there are $n=40$ users with $m = 4$ global clusters and $L = 5$ local servers and the weight vectors and item vectors are generated randomly with dimension $d = 10$. The users are selected randomly from these $10^3$ users for the experiment. Besides, we stipulate that the default privacy bugdet is $(\epsilon,\delta)=(1,0.1)$ and communication upload/download rate $U=D=1.01$.
In this section, we also compare our algorithm CDP-FCLUB-DC with the baselines listed above on movie recommendations with the MovieLens dataset.
The performances are shown in Figure \ref{fig:movielens}. Our algorithm CDP-FCLUB-DC's performance has an advantage over baseline SCLUB, CLUB, LinUCB, Homo and Homo-DC in general. Figure \ref{fig:movielens} shows CDP-FCLUB-DC performs worse than FCLUB and FCLUB-DC due to the delay communication and cluster differential privacy (CDP) as we have explained in synthetic dataset part. Different from the synthetic dataset, in the early stage, our algorithm needs more time to identify the underlying cluster structure. But after all user clusters are correctly detected at the global server, our algorithm performs better than Homo/Homo-DC that assume users are homogeneous, CLUB/SCLUB on each local server and LinUCB on each user.

\section{Conclusion and Future Work}\label{sec: conclusion}
In this paper, we formulate the federated online clustering of bandits problem, which generalizes the clustering of bandits problem to its federated counterpart. To tackle this new problem, we propose a FCLUB-CDP algorithm, which simultaneously achieves sublinear regret, sublinear communication complexity and satisfies our newly-defined clustered differential privacy requirements. Compared with benchmark algorithms, we show that FCLUB-CDP achieves superior performance regarding regret and communication cost.
There are many compelling directions for future study.
For example, it would be interesting to study our problem where local differential privacy is considered.
One could also study a more efficient protocol to further reduce the communication cost.

\subsubsection*{Acknowledgement}
The corresponding author Shuai Li is supported by National Natural Science Foundation of China (62006151). This work is sponsored by Shanghai Sailing Program.
The work of John C.S. Lui was supported in part by the RGC SRFS2122-4S02.

\clearpage
\bibliography{main}

%%%%%%%%%%%%%%%%%%%%%%%%%%%Appendix%%%%%%%%%%%%%%%%%%%%%%%%%%%%%%%%%%
\appendix
\clearpage
\onecolumn

\section*{Supplementary Material}
\section{Summary of Notations}\label{apdx:notations}
\Cref{tab:notation} summarizes the notations used through the main paper and the appendix.
\begin{table}[htbp]\caption{Table of Notations}\label{tab:notation}
\begin{center}% used the environment to augment the vertical space
% between the caption and the table
\begin{tabular}{c c p{10cm} }
\toprule
Symbol & & Meaning \\
\midrule
$n$ & & Number of users.\\
$m$ & & Number of underlying clusters.\\
$d$ & & Dimension.\\
$L$ & & Number of local servers.\\
$K$ & & Number of candidate items to be recommended each round.\\
$T$ & & Number of rounds.\\
$(\varepsilon, \delta)$ & & Privacy budgets.\\
$t$ & & Index of rounds.\\
$s$ & & Index of phases, each phase begins at $t=2^s-1$.\\
$i$ (or $i_t$) & & Index of users (or index of users at round $t$)\\
$j$ (or $j_t$) & & Index of local clusters (or index of local clusters at round $t$)\\
$k$ (or $k_t$) & & Index of global clusters (or index of global clusters at round $t$)\\
$\ell$ (or $\ell_t$) & & Index of local servers (or index of local servers at round $t$)\\
$\bD_t$ & & Candidate item set at round $t$.\\
$\btheta_i$ & & Preference vector for user $i \in [n]$.\\
$\bx$ or $\bx_t$) & & Feature vector (or the chosen vector at time $t$).\\
$\rho$ && Unknown distribution that generates $\bx$.\\
$\sigma_0^2, \sigma^2$ & & Sub-Gaussian parameters for the reward.\\
$\gamma$ & & Threshold for the between-cluster distance.\\
$\lambda_x$ & & minimal eigenvalue for the $\E[\bx\bx^{\top}]$.\\
$\cG^\ell_{s}$ & & The connection graph at local server $s$ at phase $s$, the edge connecting two users means they belong to the same cluster.\\
$P_{s,k}$ & & The $k$-th global cluster identified by the global server at phase $s$.\\
$\bV_{t,i}, \bb_{t,i}, T_{t,i}$ & & $d\times d$ Gram matrix, $d\times 1$ moment vector, number of arrival times for user $i$ at time $t$, respectively.\\
$\tilde{\bV}^{\ell}_{s, j}, \tilde{\bb}^{\ell}_{s, j}, \tilde{T}^{\ell}_{s,j}$ & & $d\times d$ Gram matrix, $d\times 1$ moment vector, total number of arrival times for all users in cluster $j$ at local server $\ell$ at the beginning of phase $s$, respectively. Used for global cluster detection.\\
$\bS^{\ell}_{t,j}, \bu^{\ell}_{t,j}, T^{\ell}_{t,j}$ & &  $d\times d$ local synchronized Gram matrix, $d\times 1$ local synchronized moment vector, local synchronized total number of arrival times for all users in cluster $j$ at local server $\ell$ at $t$, which are collected by asynchronous upload/download events, respectively. Used for recommendation.\\
$\bS^{g}_{t,k}, \bu^{g}_{t,k}, T^{g}_{t,k}$ & &  $d\times d$ global synchronized Gram matrix, $d\times 1$ global synchronized moment vector, global synchronized total number of arrival times for all users in global cluster $k$ at the global server in $t$, respectively. Used for communication.\\
$\Delta \bS^{\ell}_{t,j}, \Delta \bu^{\ell}_{t,j}, \Delta T^{\ell}_{t,j}$ & & $d\times d$ upload buffer for Gram matrix, $d\times 1$ upload buffer for moment vector, upload buffer for total number of arrival times for all users in cluster $j$ at local server $\ell$ at $t$, respectively. Used for upload.\\
$\Delta \bS^{-\ell}_{t,j}, \Delta \bu^{-\ell}_{t,j}, \Delta T^{-\ell}_{t,j}$ & & $d\times d$ download buffer for Gram matrix, $d\times 1$ download buffer for moment vector, download buffer for total number of arrival times for all users in cluster $j$ at local server $\ell$ at $t$, respectively. Used for download.\\
\bottomrule
\end{tabular}
\end{center}
\label{tab:TableOfNotationForMyResearch}
\end{table}

\begin{table}[htbp]\caption{Table of Notations (Continued)}
\begin{center}% used the environment to augment the vertical space
% between the caption and the table
\begin{tabular}{c c p{10cm} }
\toprule
Symbol & & Meaning \\
\midrule
$\bH^\ell_{t,j},\bar{\bH}^\ell_{t,j}, \tilde{\bH}^\ell_{t,j}$ & & The $d\times d$ perturbation matrix for next upload, current upload and adding right amount of $\bI$ to make $H$ positive semi-definite, respectively. Used for privacy requirements.\\
$\bh^\ell_{t,j},\bar{\bh}^\ell_{t,j}$ &  & The $d\times 1$ perturbation vector for next upload, current upload, respectively. Used for privacy requirements.\\
$c_{t,j}^{\ell}$ & & The total number of communications rounds for local cluster $j$ at server $\ell$ at before round $t$.\\
$\nu$ & & The maximum number of communications for each local server.\\
$D,U$ & & Upload and Download threshold, respectively.\\
$T_0, T_0(\alpha)$ && Number of rounds after which clusters are detected and partitioned correctly with high probability.\\
$\rho_{\max}, \rho_{\min}, \kappa$ & & Spectral bounds for the perturbation matrices $\tilde{\bH}^{\ell}_{t,j}$.\\
$\beta^\ell_{t,j}$ & & Confidence interval.\\
$R(T)$ & & Expected cumulative regrets over time $T$. \\
$C(T)$ & & Expected communication costs over time $T$.\\
$\bV_{t,j}, \bb_{t,j}, T_{t,j}$ & & The fully-synchronized gram matrix, moment vector and number of arrival times if all servers upload/download all information instantly, respectively. Used for analysis.\\
\bottomrule
\end{tabular}
\end{center}
\label{tab:TableOfNotationForMyResearchCont}
\end{table}

\section{Regret Analysis}\label{sec:regret_proof}
This section is organized as follows. In \cref{apdx_sec:high_probability_events}, we introduce several high probability events for the regret analysis. In \cref{apdx_sec:correct_partition}, we bound the time horizon after which all clusters are detected and partitioned correctly with high probability. In \cref{apdx_sec:regret_after}, we bound the regret after the clusters are detected correctly. In \cref{apdx_sec:put_together}, we put all things together to conclude \cref{thm:mainRegret}. 

\subsection{High Probability Events}\label{apdx_sec:high_probability_events}
We first define five events $B_i(\alpha), i=0, ..., 4,$ that will be helpful for the later analysis and bound the probability that each event happens.

Before we state the definition of the events, recall that $L$ is the number of local servers. $T_{t,i}$ is the number of times user $i$ comes to the system before time $t$, $T_{t,j}^{\ell}$ is the number of times users belong to local cluster $j$ at server $\ell$ comes to the system before time $t$. As defined in \cref{def:accurate}, $\rho_{\min} = \rho$,  $\rho_{\max} = 3\rho + 3\rho d\log T/\log (\min\{U,D\})$ , and $\kappa=\norm{\bar{\bh}^{\ell}_{t,j}}/\sqrt{\rho}$, where $\rho \triangleq 8\sqrt{2}\nu \log (4/\delta) (4\sqrt{d} + 2 \log (2mL/\alpha))/\varepsilon$. We denote $\rho_{0,\max}=3\rho$. 

For the gram matrices, $\bV_{t,i}$ is the gram matrix for user $i$ at time $t$, ${\tilde{\bV}^{\ell}_{t,j}}$ is the gram matrix for cluster detection at the beginning of each phase $t=2^s-1$ (Line~\ref{line:detect_cluster} in \cref{alg:cluster_detect}), $\bS^{\ell}_{t,j}$ is the synchronized gram matrix which takes information from all clusters that are in the same global cluster.

For the confidence intervals, let $\alpha \in (0,1)$ a small failure probability to be tuned later.
We denote intervals $\beta_i(T_{t,i},\alpha)=\sigma_{0}\sqrt{2\log(\frac{1}{\alpha})+d\log(1+\frac{T_{t,i}}{\lambda d})}+\sqrt{\lambda}$, $\beta^{\ell}_j(\tilde{T}^{\ell}_{t,j},0, \alpha)=\sigma_{0}\sqrt{2\log(\frac{1}{\alpha})+d\log(3+\frac{\tilde{T}^{\ell}_{t,j}}{\rho_{\min} d})}+\sqrt{\rho_{0,\max}}+\kappa$, $\beta^{\ell}_j(T^{\ell}_{t,j},L, \alpha)=\sigma_{0}\sqrt{2\log(\frac{1}{\alpha})+d\log(\frac{\rho_{\max}}{\rho_{\min}}+\frac{T^{\ell}_{t,j}}{\rho_{\min} d})}+\sqrt{L\rho_{\max}}+\sqrt{L}\kappa$.

Now we are ready to define a series of events as follows.

$\cB_{0}(\alpha)=\left\{(\rho_{\min}, \rho_{\max}, \kappa) \text{ are }\alpha/(mL)\text{-accurate bounds for } \tilde{\bH}^{\ell}_{t,j}, \text{ for all } t\in[T], j \in [m], \ell \in [L]\right\}$.

$\cB_{1}(\alpha)=\left\{\norm{\btheta_i-\hat{\btheta}_{t,i}}_{\lambda \bI + \bV_{t,i}} \le \beta_i(T_{t,i}, \alpha/n) \text{ for all } t\in [T], i \in [n] \right\}$.

$\cB_{2}(\alpha)=\left\{\lambda_{\min}(\bV_{t,i}) \ge T_{t,i} \lambda_x /8, \text{ for all } T_{t,i} \ge \frac{1024}{\lambda_x^2}\log \frac{512d}{\lambda_x^2 \alpha/n}, \text{ for all } i\in [n]\right\}$.

$\cB_{3}(\alpha)=\left\{\norm{\btheta_j - \hat{\btheta}^{\ell}_{t,j}}_{\tilde{\bV}^{\ell}_{t,j}} \le \beta^{\ell}_j(\tilde{T}^{\ell}_{t,j}, 0, \alpha/(mL)) \text{ for all } t\in[T], j \in [m], \ell \in [L]\right\}$.

$\cB_{4}(\alpha)=\left\{\norm{\btheta^*-\hat{\btheta}^{\ell}_{t,j}}_{\bS^{\ell}_{t,j}} \le \beta^{\ell}_j(T^{\ell}_{t,j}, L, \alpha/(mL)) \text{ for all } t\in[T], j \in [m], \ell \in [L]\right\}$.

For each of the event $B_i(\alpha), i=0, ..., 4$, it is noted that their probability $\Pr\{B_i(\alpha)\}\ge 1-\alpha$, where $B_0(\alpha)$ is given by the definition in \Cref{eq:alpha_accurate} and the tree-based privacy protocol in \cref{sec:algorithm}, $B_1(\alpha)$ is by \citep[Theorem 2]{abbasi2011improved}, $B_2(\alpha)$ is by \citep[Claim 1]{gentile2014online} and \citep[Lemma 7]{li2018online}, $B_3(\alpha)$ and $B_4(\alpha)$ are by \Cref{lemma:cdp_beta} whose proofs are postponed to \cref{apdx_sec:regret_after}.
\subsection{Correctness of the Cluster Detection}\label{apdx_sec:correct_partition}
In this section, we show in \cref{cdp_to} when $T>2T_0(\alpha)$, then with high probability, all local and global clusters are correctly detected and partitioned. In \cref{cdp_t0_rounds}, we give the explicit formulation for $T_0(\alpha)$.
\begin{lemma}\label{cdp_to}
Let $\bar{T}_i(\alpha)=\max\Big\{A(\alpha), B, C, D(\alpha), E(\alpha)\Big\}$, $A(\alpha)=\frac{(96\sigma_{0})^{2}\log(\frac{2}{\alpha})}{\lambda_{x}\gamma^{2}}$, $B=\frac{(96\sigma_{0})^{2}d}{\lambda_{x}\gamma^{2}}\log(\frac{4608\sigma_{0}^{2}}{\lambda_{x}\gamma^{2}})$, $C=\frac{192^2\times6\sqrt{2}\nu(\log(\frac{
4}{\delta}))\sqrt{d}}{\varepsilon\lambda_{x}\gamma^{2}}$, $D(\alpha)=\frac{192^2\times3\sqrt{2}\nu(\log\frac{
4}{\delta})\log(\frac{2mL}{\alpha})} {\varepsilon\lambda_{x}\gamma^{2}}$, $E(\alpha)=\frac{1024}{\lambda_x^2}\log \frac{512nd}{\lambda_x^2 \alpha}$.
When $T>2T_{0}(\alpha) \triangleq  2\times(16n\log(\frac{T}{\alpha}) + 4n\bar{T}_i(\alpha)) $,
then with probability $1-6\alpha$, (a): 
All local clusters are partitioned correctly;
% \begin{equation}
%     \norm{\btheta_{i} - \hat{\btheta}_{t,i}} < \frac{\gamma}{4},
% \end{equation}
% where $\hat{\btheta}_{t,i}=\hat{\btheta}_{t, i}=(\lambda \bI + \b\bV_{t-1, i})^{-1}\bb_{t-1,i}$.

(b) All local clusters are correctly merged at the global server.

% (b):Fix any local cluster $C_j$, when it is correctly  let $T_{t,j}=\sum_{i \in C_j}$, 
% \begin{equation}
%     \norm{\btheta_{j} - \hat{\btheta}_{t,j}} < \frac{\gamma}{4},
% \end{equation}
% where $\hat{\btheta}^{\ell}_{t, j}= (\tilde{\bV}^\ell_{s, j})^{-1}\tilde{b}^l_{s,j}$ and $F(x)=\sqrt{\frac{1+\ln (1+x)}{(1+x)}}$

% Let $\rho_{min} = \sqrt{32}\log_{2}(T_{i}+1)2\log(\frac{4}{\delta})(4\sqrt{d}+2\ln(\frac{2T_{i}}{\alpha}))/\varepsilon $, $\rho_{\max} = 3\rho_{min}$, and $\kappa_{t} =\sqrt{\log_{2}(T_{i}+1)(\sqrt{d}+2\log(2T_{i}/\alpha))/\sqrt{2\varepsilon}} $, then 
% \[
% \norm{\btheta_{i} - \tilde{\btheta_{t,i}}} < \frac{\gamma}{4}
% \] can be satisfied when$T_{i} > \max\left\{\frac{9216\sigma_{0}^{2}\log(\frac{2}{\alpha})}{\gamma^{2}\lambda_{x}}, \frac{9216\sigma_{0}^{2}d}{\lambda_{x}\gamma^{2}}\log(\frac{4608\sigma_{0}^{2}d}{\lambda_{x}\gamma^{2}}),\\ \frac{27648\times\sqrt{32}L^{2}\bS^{2}2(\log(\frac{
% 4}{\delta}))\sqrt{d}}{\varepsilon\lambda_{x}\gamma^{2}}\log(\frac{13824\times\sqrt{32}L^{2}\bS^{2}2(\log(\frac{
% 4}{\delta}))\sqrt{d}}{\varepsilon\lambda_{x}\gamma^{2}}),
% \frac{55296\times\sqrt{32}L^{2}\bS^{2}2\log(\frac{4}{\delta})}{\lambda_{x}\varepsilon\gamma^{2}}\log^{2}(\frac{\sqrt{48}\times288\times32^{\frac{1}{4}}LS\sqrt{2\log(4/\delta)}}{\sqrt{\lambda_{x}\varepsilon}\gamma}), \frac{1152\sqrt{2}L^{2}2\sqrt{d}}{\varepsilon\lambda_{x}\gamma^{2}}\log_{2}(\frac{576\sqrt{2}L^{2}2\sqrt{d}}{\varepsilon\lambda_{x}\gamma^{2}}), \frac{25088L^{2}2}{\varepsilon\lambda_{x}\gamma^{2}}\log^{2}(\frac{28\times\sqrt{8}L(\sqrt{2})}{\sqrt{\varepsilon\lambda_{x}\gamma^{2}}})\right\}$
\begin{proof}
For any user $i$, it suffices to show that when $T_{t,i}$ is larger than some threshold $\bar{T}_{i}(\alpha)$ (whose value will be settled later), then we have the following two results hold.

(a): Under event $\cB_1(\alpha)$ and $\cB_2(\alpha)$, which happens at least $1-2\alpha$,
\begin{equation}\label{ieq:user_level_correct}
    \norm{\btheta_{i} - \hat{\btheta}_{t,i}}_2\le \frac{\norm{\btheta_{i} - \hat{\btheta}_{t,i}}_{\lambda \bI + \bV_{t,i}}}{\sqrt{\lambda_{\min}(\lambda \bI + \bV_{t,i})}} \le \frac{\beta_i(T_{t,i}, \alpha/n)}{\sqrt{\lambda+T_{t,i}\lambda_x/8}}\le \frac{\sigma_{0}\sqrt{2\log(\frac{n}{\alpha})+d\log(1+\frac{T_{t,i}}{\lambda d})}+\sqrt{\lambda}}{\sqrt{T_{t,i}\lambda_x/8}}  < \frac{\gamma}{4}
\end{equation}
where the last inequality is valid when $T_{t,i} \ge \bar{T}_i(\alpha)$.

When \Cref{ieq:user_level_correct} holds, for any two user $i_1, i_2$ who belong to different clusters, i.e., $\norm{\btheta_{i_1}-\btheta_{i_2}}\ge \gamma$,  $\norm{\hat{\btheta}_{i_1}-\hat{\btheta}_{i_1}}\ge \norm{\btheta_{i_1}-\btheta_{i_2}}_2-\norm{\hat{\btheta}_{i_2}-\btheta_{i_2}}- \norm{\hat{\btheta}_{i_1}-\btheta_{i_1}}> \gamma/2 \ge  \frac{\beta_i(T_{t,i_1}, \alpha/n)}{\sqrt{\lambda+T_{t,i_1}\lambda_x/8}} + \frac{\beta_i(T_{t,i_2}, \alpha/n)}{\sqrt{\lambda+T_{t,i_2}\lambda_x/8}}$, which will trigger the condition in Line~\ref{line:detect_delete} in \Cref{alg:cluster_detect}.
On the other hand, when the condition in Line~\ref{line:detect_delete} in \Cref{alg:cluster_detect} holds, 
$\frac{\beta_i(T_{t,i_1}, \alpha/n)}{\sqrt{\lambda+T_{t,i_1}\lambda_x/8}} + \frac{\beta_i(T_{t,i_2}, \alpha/n)}{\sqrt{\lambda+T_{t,i_2}\lambda_x/8}} < \norm{\hat{\btheta}_{i_1}-\hat{\btheta}_{i_1}}\le \norm{\btheta_{i_1}-\btheta_{i_2}}_2+\norm{\hat{\btheta}_{i_2}-\btheta_{i_2}}+  \norm{\hat{\btheta}_{i_1}-\btheta_{i_1}}\le \norm{\btheta_{i_1}-\btheta_{i_2}}+ \frac{\beta_i(T_{t,i_1}, \alpha/n)}{\sqrt{\lambda+T_{t,i_1}\lambda_x/8}} + \frac{\beta_i(T_{t,i_2}, \alpha/n)}{\sqrt{\lambda+T_{t,i_2}\lambda_x/8}}$, which implies $\norm{\btheta_{i_1}-\btheta_{i_2}}> 0$.
In other words, all local clusters are correct after $T_{t,i} \ge \bar{T}_i(\alpha)$ for all users.

(b): Moreover, under event $\cB_0(\alpha),\cB_2(\alpha)$ and $\cB_3(\alpha)$, for any correctly partitioned local cluster $C_j$ at server $\ell$ with true parameter $\btheta_j$, with probability at least $1-3\alpha$,
\begin{equation}\label{ieq:cluster_level_correct}
    \norm{\btheta_{j} - \hat{\btheta}^{\ell}_{t,j}}\le \frac{\norm{\btheta_{j} - \hat{\btheta}^{\ell}_{t,j}}_{\tilde{\bV}^\ell_{t, j}}}{\sqrt{\lambda_{\min}(\tilde{\bV}^\ell_{t, j})}} \le \frac{\beta^{\ell}_j(\tilde{T}^{\ell}_{t,j}, 0, \alpha/(mL))}{\sqrt{\rho_{\min}+\tilde{T}^{\ell}_{t,j}\lambda_x/8}}\le \frac{\sigma_{0}\sqrt{2\log(\frac{mL}{\alpha})+d\log(3+\frac{\tilde{T}^{\ell}_{t,j}}{\rho_{\min} d})}+\sqrt{\rho_{0,\max}}+\kappa}{\sqrt{\tilde{T}^{\ell}_{t,j}\lambda_x/8}}  < \frac{\gamma}{4}
\end{equation},
where $\tilde{T}^{\ell}_{t,j}=\sum_{i \in C_j}T_{t, i}$ and $\hat{\btheta}_{t,j}= (\tilde{\bV}^\ell_{t, j})^{-1}\tilde{\bb}^l_{t,j}$.

Here $\beta^{\ell}_j(\tilde{T}^{\ell}_{t,j}, 0, \alpha/(mL))$ is irrelevant to $L$ or $m$ (except probability $\alpha/(mL))$ because the merge operation only uses the the local observations from server $\ell$ with local noise added at the beginning of each phase, but doesn't contain any information from other local servers. 
According to similar argument, we can show for any two cluster $j_1,j_2$ who belong to the same cluster, i.e., $\norm{\btheta_{j_1}-\btheta_{j_2}}= 0$, we have $\norm{\hat{\btheta}_{j_1}-\hat{\btheta}_{j_2}}\le  \frac{\beta^{\ell}_j(\tilde{T}^{\ell}_{t,j_1}, 0, \alpha/(mL))}{\sqrt{\rho_{\min}+\tilde{T}^{\ell}_{t,j_1}\lambda_x/8}} + \frac{\beta^{\ell}_j(\tilde{T}^{\ell}_{t,j_2}, 0, \alpha/(mL))}{\sqrt{\rho_{\min}+\tilde{T}^{\ell}_{t,j_2}\lambda_x/8}}$, which will trigger the condition in Line~\ref{line:detect_merge} in \Cref{alg:cluster_detect}.
On the other hand, if the above condition is triggered, then it implies $\norm{\btheta_{j_1}-\btheta_{j_2}}< \gamma$. In other words, all global clusters are merged correctly.
% \[\norm{\btheta_{i} - \tilde{\btheta_{t,i}}} \le \frac{\beta_{T_{i}}}{\sqrt{\lambda+T_{i}\lambda_{x}/8}} \le 
% \frac{\beta(T_{i})}{\sqrt{T_{i}\lambda_{x}/8}},
% \]

Now for the last inequality of \Cref{ieq:user_level_correct} and \Cref{ieq:cluster_level_correct}, we only need to prove the following inequality holds for any $i \in [n]$, 
\begin{equation}\label{ieq:sufficient_correct}
    \frac{\sigma_{0}\sqrt{2\log(\frac{n}{\alpha})+d\log(3+\frac{T_{t,i}}{ d})}+\sqrt{\rho_{0,\max}}+\kappa}{\sqrt{T_{t,i}\lambda_x/8}}  \le \frac{\gamma}{4}
\end{equation}
, which uses the observation that $\rho_{0,\max} > \lambda \ge 1$ and $\kappa_t > 0$ for \Cref{ieq:user_level_correct} and $n \ge mL, \lambda_{0,\min} \ge 1, \tilde{T}^{\ell}_{t,j}=\sum_{i \in C_j}T_{t, i}\ge T_{t,i}$ for \Cref{ieq:cluster_level_correct}. 

After we prove the sufficient $\bar{T}_{i}(\alpha)$ as shown in \Cref{cdp_t0_rounds} so that \Cref{ieq:sufficient_correct} holds,
now by \citep[Lemma 8]{li2018online}, at global time $T_0(\alpha) = 16n\log(\frac{T}{\alpha}) + 4n\bar{T}_i(\alpha)$, \Cref{ieq:sufficient_correct} is correct with probability at least $1-\alpha$. Then at the next cluster detection (which occurs at most two times of $T_0(\alpha)$), the global server will partition the global clusters correctly with probability at least $1-6\alpha$ (by using union bounds of all corresponding events).

\end{proof}
\end{lemma}

The following lemma states a threshold after which \Cref{ieq:sufficient_correct} will hold.
\begin{lemma}\label{cdp_t0_rounds}
When $T_{t,i} > \bar{T}_i(\alpha)=\max\Big\{A(\alpha), B, C, D(\alpha), E(\alpha)\Big\}$, it holds that
\begin{equation}
\frac{\sigma_{0}\sqrt{2\log(\frac{n}{\alpha})+d\log(3+\frac{T_{t,i}}{d})}+\sqrt{\rho_{0,\max}}+\kappa}{\sqrt{T_{t,i}\lambda_{x}/8}} < \frac{\gamma}{4}
\end{equation}
, where 
 $A(\alpha)=\frac{(96\sigma_{0})^{2}\log(\frac{2}{\alpha})}{\lambda_{x}\gamma^{2}}$, $B=\frac{(96\sigma_{0})^{2}d}{\lambda_{x}\gamma^{2}}\log(\frac{4608\sigma_{0}^{2}}{\lambda_{x}\gamma^{2}})$, $C=\frac{192^2\times6\sqrt{2}\nu(\log(\frac{
4}{\delta}))\sqrt{d}}{\varepsilon\lambda_{x}\gamma^{2}}$, $D(\alpha)=\frac{192^2\times3\sqrt{2}\nu(\log\frac{
4}{\delta})\log(\frac{2mL}{\alpha})} {\varepsilon\lambda_{x}\gamma^{2}}$, $E(\alpha)=\frac{1024}{\lambda_x^2}\log \frac{512nd}{\lambda_x^2 \alpha}$.
% $T_{i}(\alpha) = \max\Big\{\frac{9216\sigma_{0}^{2}\log(\frac{2}{\alpha})}{\gamma^{2}\lambda_{x}}, \frac{9216\sigma_{0}^{2}d}{\lambda_{x}\gamma^{2}}\log(\frac{4608\sigma_{0}^{2}d}{\lambda_{x}\gamma^{2}}), A(\delta), B(\delta), \frac{1152\sqrt{2}\times2\sqrt{d}}{\varepsilon\lambda_{x}\gamma^{2}}\log_{2}(\frac{576\sqrt{2}\times2\sqrt{d}}{\varepsilon\lambda_{x}\gamma^{2}}), \frac{25088\times2}{\varepsilon\lambda_{x}\gamma^{2}}\log^{2}(\frac{28\times4}{\sqrt{\varepsilon\lambda_{x}\gamma^{2}}})\Big\}$
% Here $A(\delta)=\frac{55296\times\sqrt{32}\times4\log(\frac{4}{\delta})}{\lambda_{x}\varepsilon\gamma^{2}}\log^{2}(\frac{96\sqrt{6}\times32^{\frac{1}{4}}\sqrt{\log(4/\delta)}}{\sqrt{\lambda_{x}\varepsilon}\gamma})$,$B(\delta)=\frac{27648\times\sqrt{32}\times2(\log(\frac{
%  4}{\delta}))\sqrt{d}}{\varepsilon\lambda_{x}\gamma^{2}}\log(\frac{13824\times\sqrt{32}\times2(\log(\frac{
%  4}{\delta}))\sqrt{d}}{\varepsilon\lambda_{x}\gamma^{2}})$
\begin{proof}
We can divide this proof into three parts: $\frac{\sigma_{0}\sqrt{2\log(\frac{n}{\alpha})+d\log(3+\frac{T_{t,i}}{d})}}{\sqrt{T_{t,i}\lambda_{x}/8}} < \frac{\gamma}{12}$,$\frac{\sqrt{\rho_{0,max}}}{\sqrt{T_{t,i}\lambda_{x}/8}} < \frac{\gamma}{12}$ and $\frac{\kappa}{\sqrt{T_{t,i}\lambda_{x}/8}} < \frac{\gamma}{12}$.

The first part can be satisfied by (a)$\frac{\sigma_{0}\sqrt{2\log(n/\alpha)}}{\sqrt{T_{t,i}\lambda_{x}/8}}< \frac{\gamma}{24}$ and (b) $\frac{\sigma_{0}\sqrt{d\log(3+T_{t,i}/d)}}{\sqrt{T_{t,i}\lambda_{x}/8}} < \frac{\gamma}{24}$. And (a) can be satisfied by $T_{t,i}>\frac{(96\sigma_{0})^{2}\log(\frac{2}{\alpha})}{\gamma^{2}\lambda_{x}}$, (b) by \citep[Lemma 9]{li2018online} can be satisfied by $T_{t,i} > \frac{(96\sigma_{0})^{2}d}{\lambda_{x}\gamma^{2}}\log(\frac{4608\sigma_{0}^{2}}{\lambda_{x}\gamma^{2}})$.

Then consider the second part. Since $\rho_{0,max} \le 24\sqrt{2}\nu\log(\frac{4}{\delta})(4\sqrt{d}+2\log(\frac{2mL}{\alpha}))/\varepsilon$, we can prove this part by (c) $\frac{96\times 8\sqrt{2}\nu\log(\frac{4}{\delta})(\sqrt{d}) }{T_{t,i}\varepsilon\lambda_{x}} < \frac{\gamma^{2}}{2\times 12^2}$ and (d) $\frac{48\times8\sqrt{2}\nu\log(\frac{4}{\delta})(\log(\frac{mL}{\alpha})) }{T_{t,i}\varepsilon\lambda_{x}} < \frac{\gamma^{2}}{2\times 12^2}$.
By some math calculation, (c) can be satisfied by $T_{t,i} > \frac{192^2\times6\sqrt{2}\nu(\log(\frac{
4}{\delta}))\sqrt{d}}{\varepsilon\lambda_{x}\gamma^{2}}$ and (d) can be satisfied when  $T_{t,i} > \frac{192^2\times3\sqrt{2}\nu(\log\frac{
4}{\delta})\log(\frac{2mL}{\alpha})} {\varepsilon\lambda_{x}\gamma^{2}}$.
% As for (d) , since $\alpha$ could be as small as $\frac{1}{T}$,and $\log_{2}(T_{i})$ and $\ln(\sqrt{2}T_{i})$ are close, we can simplify it to $\frac{2\times24\times\sqrt{32}\times2\log^{2}(T_{i})2\log(\frac{4}{\delta})}{T_{i}\varepsilon\lambda_{x}} < \frac{\gamma^{2}}{288}$, which is equal to $\frac{96\times\sqrt{6}\times32^{\frac{1}{4}}\log(T_{i}^{\frac{1}{2}})\sqrt{\log{4/\delta}}}{\sqrt{\lambda_{x}\varepsilon}\gamma} < \sqrt{T_{i}}$.Then by \citep[Lemma9]{li2018online} we can get $T_{i} > \frac{55296\times\sqrt{32}\times4\log(\frac{4}{\delta})}{\lambda_{x}\varepsilon\gamma^{2}}\log^{2}(\frac{96\sqrt{6}\times32^{\frac{1}{4}}\sqrt{\log(4/\delta)}}{\sqrt{\lambda_{x}\varepsilon}\gamma})$.

Finally we prove the third part, since $\kappa = \sqrt{2\nu(\sqrt{d}+2\log(2mL/\alpha))/(\sqrt{2}\varepsilon)}$ , the third part has a similar form to the second part, so their proof are similar. We can get the third part could be satisfied by $T_{t,i} > \frac{48^2\sqrt{2}\nu\sqrt{d}}{\varepsilon\lambda_{x}\gamma^{2}}$ and $T_{t,i} >  \frac{48^2\times 2\sqrt{2}\nu\log (\frac{2mL}{\alpha})}{\varepsilon\lambda_{x}\gamma^{2}}$.

Considering the $T_{t,i} \ge\frac{1024}{\lambda_x^2}\log \frac{512d}{\lambda_x^2 \alpha/n} $ required by $B_1(\alpha)$, and put all these together the \cref{cdp_t0_rounds} is proved.
\end{proof}
\end{lemma}

\subsection{Regret Bound After Clusters are Correctly Detected}\label{apdx_sec:regret_after}
In this section, we prove the regret upper bound after the clusters are correctly detected and partitioned. Such a bound relies on the confidence interval given by \cref{lemma:cdp_beta}.
\begin{proposition}\label{thm:after_clustering_regret}
After $2T_{0}(\alpha/8)$ rounds, with probability at least $(1-\frac{3\alpha}{4})$, the clusters on both local servers and global servers are correctly partitioned, and the partition won't change in later rounds. Then fix any true global cluster $j \in [m]$ whose local clusters are scattered on $L' \le L$ servers, let users belong to $j$ appear total $T_j$ times, then with probability $(1-\frac{\alpha}{4})$, the total regret for true cluster $j$ is bounded by
\begin{align*}
R_j(T_j) &\le (\sigma_{0}\sqrt{2\log(\frac{8mL}{\alpha})+d\log(\frac{\rho_{\max}}{\rho_{\min}}+\frac{T_{j}}{d\rho_{\min}})}+\sqrt{L\rho_{\max}}+\kappa\sqrt{L})(\sqrt{d\log(1+\frac{T}{d\rho_{\min}})})\Gamma\sqrt{T_j} \\
&\le  O\Big(\sqrt{L\frac{d\log T}{\log \min\{U,D\}} \frac{\log\log T \log (1/\delta) (\sqrt{d}+\log T)}{\varepsilon}}\Gamma\sqrt{d\log T }\sqrt{T_j}\Big)
\end{align*}
and the total regret is bounded by
\begin{align}\label{apdx_eq:prop1}
    R(T) &\le \sum_{j \in [m]}R_j(T_j)\le \tilde{O}\Big(dL\sqrt{mT\frac{\log(1/\delta)}{\varepsilon}}\log^{1.5} T \Big)
\end{align}
where $\Gamma=\sqrt{D(1+(L'-1)(U-1))+ U-1}$ and $ \sqrt{d} < \log T $.
\end{proposition}

\begin{proof}
We consider the case when all true clusters are correctly detected and partitioned both locally and globally.
Suppose the user belong to global cluster $j$ whose underlying preference vectors are $\btheta^*$ comes to the server at time slots $\{t_1, ..., t_{T_j}\}$.
For this true cluster $j$, users belong to $j$ lies in $L'$ servers to form local clusters be $\{(\ell_1,j_1), ..., (\ell_{L'}, j_{L'})\}$.
With a little abuse of notation, we denote the perturbation at time $\tau$ $\sum_{p=1}^{L'}\tilde{\bH}^{\ell_p}_{\tau, j_p}$ and $\sum_{p=1}^{L'}\bar{\bh}^{\ell_p}_{\tau, j_p}$ by $\tilde{\bH}_{\tau}$ and $\bar{\bh}_{\tau}$, respectively.
At time $t=t_k, k \in [T_j]$, user $i_{t}$ in server $\ell_{t}$ and local cluster $j_{t}$ arrives, with candidate feature vector sets $\bD_t=(\bx_{t, 1}, ..., \bx_{t,K})$.
With a little bit abuse of the notation, let $\hat{\btheta}_{t,i_t}\triangleq \hat{\btheta}^{\ell_t}_{t,j_t}$
Recall that $(\tilde{\btheta}_{t},\bx_t)=\argmax_{(\btheta,\bx) \in \cC_t \times \bD_t}\prdct{\btheta, \bx}$ is the action selected by our algorithm and $\cC_t=\{\btheta\in \R^d: \norm{ \hat{\btheta}_{t,i_t}-\btheta}_{\bS^{\ell_t}_{t,j_t}} \le \beta^{\ell_t}_{t,j_t}\}$ as defined in \cref{lemma:cdp_beta}.
Denote the full-synchronized gram matrix from all servers for cluster $j$ without any perturbation as $\bV_{t,j}=\sum_{\tau=t_k, k \in [T_j]}\bx_{\tau} \bx_{\tau}^{\top} + \tilde{\bH}_{t}$ (same for $b_{t,j}$ and $T_{t,j}$).
Let $\bx^{*}=\argmax_{\bx \in \bD_t}\prdct{\btheta^{*},\bx_{t}}$.
The instantaneous regret $r_{t,j}$ for cluster $j$ can be written as:

\begin{align}
r_{t,j}&=\prdct{\btheta^{*},\bx^*}-\prdct{\btheta^{*},\bx_t}\notag\\
&\le \prdct{\tilde{\btheta}_t, \bx_t} - \prdct{\btheta^{*}, \bx_t} \label{eq:apdx_alg_select}\\
&=\prdct{\tilde{\btheta}_t-\hat{\btheta}_{t,i_t}, \bx_t} + \prdct{\hat{\btheta}_{t,i_t}-\btheta^*, \bx_t}\notag\\
&\le \norm{\tilde{\btheta}_t-\hat{\btheta}_{t,i_t}}_{\bV_{t,j}} \norm{\bx_t}_{\bV_{t,j}^{-1}} + \norm{\hat{\btheta}_{t,i_t}-\btheta^*}_{\bV_{t,j}} \norm{\bx_t}_{\bV_{t,j}^{-1}} \label{eq:apdx:cauchy_1}\\
&=  \norm{\tilde{\btheta}_t-\hat{\btheta}_{t,i_t}}_{\bS^{l_t}_{t,j_t}} \norm{\bx_t}_{V^{-1}_{t}} \frac{\norm{\tilde{\btheta}_t-\hat{\btheta}_{t,i_t}}_{\bV_{t,j}}}{\norm{\tilde{\btheta}_t-\hat{\btheta}_{t,i_t}}_{\bS^{l_t}_{t,j_t}}} + \norm{\hat{\btheta}_{t,i_t}-\btheta^*}_{\bS^{l_t}_{t,j_t}} \norm{\bx_t}_{\bV_{t,j}^{-1}} \frac{\norm{\hat{\btheta}_{t,i_t}-\btheta^*}_{\bV_{t,j}} }{\norm{\hat{\btheta}_{t,i_t}-\btheta^*}_{\bS^{l_t}_{t,j_t}}}\notag\\
&\le 2P_t\beta_{t, j_t}^{l_t} \norm{\bx_t}_{\bV_{t,j}^{-1}}\le 2\Gamma\beta_{t, j_t}^{l_t} \norm{\bx_t}_{\bV_{t,j}^{-1}},\label{eq:apdx_instant_regret}
\end{align}
where \Cref{eq:apdx_alg_select} is because the definition of the selected item $\bx_t$, \Cref{eq:apdx:cauchy_1} is by Cauchy Schwarz inequality and \Cref{eq:apdx_instant_regret} is by \Cref{lemma:delay}.

By summing over all instant regrets, we have
\begin{align}\label{eq:apdx_total_regret}
    R_j(T_j)&\le \sum_{t \in [T_j]} r_{t,j}\\
    &\le 2\Gamma \beta^{\ell_t}_{t,j_t} \sum_{t \in [T_j]}\norm{\bx_t}_{\bV_{t,j}^{-1}}\notag\\
    &\le 2\Gamma \beta^{\ell}_j(T, L, \alpha/(mL)) \sqrt{T_j\sum_{t \in [T]}\norm{\bx_t}^2_{\bV_{t,j}^{-1}}}\label{apdx_eq:cauchy_sw2}\\
    &\le 2\Gamma \beta^{\ell}_j(T, L, \alpha/(mL)) \sqrt{T_j}\sqrt{d\log (1+ \frac{T}{d\rho_{\min}})}\label{eq:apdx_potential}\\
    &\le (\sigma_{0}\sqrt{2\log(\frac{8mL}{\alpha})+d\log(\frac{\rho_{\max}}{\rho_{\min}}+\frac{T_{j}}{d\rho_{\min}})}+\sqrt{L\rho_{\max}}+\kappa\sqrt{L})(\sqrt{d\log(1+\frac{T}{d\rho_{\min}})})\Gamma\sqrt{T_j}\label{eq:apdx_lemma4}
\end{align}
where \Cref{apdx_eq:cauchy_sw2} is due to the Cauchy-Schwarz inequality, \Cref{eq:apdx_potential} is by \citep[Lemma 22]{shariff2018differentially} and \Cref{eq:apdx_lemma4} is by \cref{lemma:cdp_beta}.

% By \cite{shariff2018differentially}Lemma22 we have  $\log(\frac{\det(\bV_{T})}{\det(\bV_{1})}) \le d\log(\frac{tr U_{1}+T_{j}}{d\det^{1/d}\bV_{1}}) = d\log(3mL+\frac{T_{j}}{d\rho_{min}}) $, and combine the upper bound of $\beta_{t}$ and $P_{t}$, the proof can be completed.
\end{proof}

The following lemma bounds the ratio caused by using the delayed information $\bS_{t,j_t}^{\ell_t}$ instead of the fully-synchronized information $\bV_{t,j}$.
\begin{lemma}\label{lemma:delay}
Let $P_{t} = \max_{\btheta \neq 0}\sqrt{\frac{\btheta^{\top}\bV_{t,j}\btheta}{\btheta^{\top}\bS^{l_t}_{t, j_t}\btheta}}$, then
\begin{equation}
    P_{t} \le\Gamma\triangleq\sqrt{D(1+(L'-1)(U-1))+ U-1}
\end{equation}
\end{lemma}
\begin{proof}
Without loss of generality, let the local cluster ids are $(l_1, j_1),..., (l_{L'}, j_{L'})$ with $(l_t,j_t)=(l_1, j_1)$ and the global gram matrix is $\bS^{g}_{t,k_t}$, we have
\begin{align}
    \frac{\btheta^{\top} \bV_{t,j} \btheta}{\btheta^{\top} \bS^{l_1}_{t, j_1} \btheta } &= \frac{\btheta^{\top} \bS^{g}_{t, j_t}\btheta + \sum_{k=1}^{L'} \btheta^{\top}\Delta \bS^{l_k}_{j_k,t} \btheta}{\btheta^{\top} \bS^{l_1}_{t, j_1} \btheta}\label{apdx_eq:fully_delay_relation}\\
    &=  \frac{\btheta^{\top} \bS^{g}_{t, k_t}\btheta}{\btheta^{\top} \bS^{l_1}_{t, j_1}\btheta}  + \frac{\btheta^{\top} \Delta \bS^{l_1}_{t,j_1} \btheta}{\btheta^{\top}\bS^{l_1}_{t, j_1}\btheta} + \sum_{k=2}^{L'}\frac{\btheta^{\top} \Delta \bS^{l_k}_{t,j_k} \btheta}{\btheta^{\top} \bS^{l_k}_{t, j_k}\btheta} \frac{\btheta^{\top} \bS^{l_k}_{t,j_k} \btheta}{\btheta^{\top} \bS^{l_1}_{t, j_1}\btheta}\notag\\
    &\le \frac{\btheta^{\top} \bS^{g}_{t, k_t}\btheta}{\btheta^{\top} \bS^{l_1}_{t, j_1}\btheta} + \frac{\btheta^{\top} \Delta \bS^{l_1}_{t,j_1} \btheta}{\btheta^{\top}\bS^{l_1}_{t, j_1}\btheta} + \sum_{k=2}^{L'}\frac{\btheta^{\top} \Delta \bS^{l_k}_{t,j_k} \btheta}{\btheta^{\top} \bS^{l_k}_{t, j_k}\btheta} \frac{\btheta^{\top} \bS^{l_k}_{t,j_k} \btheta}{\btheta^{\top} \bS^{g}_{t, k_t}\btheta} \label{eq:apdx_pt0}\\
    &\le \frac{\det (\bS^{g}_{t, k_t}) }{\det (\bS^{l_1}_{t, j_1})} + (\frac{\det (\bS^{l_1}_{t, j_1} + \Delta \bS^{l_1}_{t,j_1}) }{\det (\bS^{l_1}_{t, j_1})}-1) + \sum_{k=2}^{L'}(\frac{\det (\bS^{l_k}_{t, j_k} + \Delta \bS^{l_k}_{t,j_k}) }{\det (\bS^{l_k}_{t, j_k})}-1) \frac{\det(\bS^{l_k}_{t, j_k})}{\det(\bS^{g}_{t, k_t})}\label{eq:apdx_pt1}\\
    &\le D + (U-1) + (L'-1)D (U-1)\label{eq:apdx_pt2},
\end{align}

where \Cref{apdx_eq:fully_delay_relation} is due to $\bV_{t,j}=\bS_{t,j_t}^{g}+\sum_{k=1}^{L'} \btheta^{\top}\Delta \bS^{l_k}_{j_k,t} \btheta$,
\Cref{eq:apdx_pt0} is because $\bS_{t,j_1}^{\ell_k}\preccurlyeq \bS_{t,k_t}^g $,
 \Cref{eq:apdx_pt1} is due to the fact that $\sup_{x\neq 0}\frac{\bx^{\top}\bA\bx}{\bx^{\top}\bB\bx} \le \frac{\det(\bA)}{\det(\bB)}$ for any positive semi-definite matrices $\bA, \bB, \bC$ s.t. $\bA=\bB+\bC$ \citep[Lemma 12]{abbasi2011improved} and \Cref{eq:apdx_pt2} is because of the condition of upload \Cref{eq:upload_condition} and download \Cref{eq:download_condition}.

\end{proof}
% \begin{lemma}\label{r_{t}}
% The regret generated in round t satisfies
% \[r_{t} \le 2\beta_{t}\norm{x_{it}}_{\bV_{t}^{-1}}\sqrt{P_{t}}
% \]
% \begin{proof}
% \[
% \begin{aligned}
% r_{t} &= \btheta^{*\top}x_{*} -\btheta^{*\top}x_{it}\\
% &\le \tilde{\btheta\top}x_{it} - \btheta^{*\top}x_{it}\qquad (\text{assume} <\tilde{\btheta},x_{it}> = \argmax_{\btheta \in C_{t}, x \in D_{t}}<\btheta,x>)\\
% &=  (\tilde{\btheta}-\hat{\btheta}_{\bV_{it,t}})^{\top}x_{it}+(\hat{\btheta}_{\bV_{it,t}}-\btheta^{*})^{\top}x_{it}\\
% &\le \norm{\tilde{\btheta}-\hat{\btheta}_{\bV_{it,t}}}_{\bV_{t}}\norm{x_{it}}_{\bV_{t}^{-1}} + \norm{\hat{\btheta}_{\bV_{it,t}}-\btheta^{*}}_{\bV_{t}}\norm{x_{it}}_{\bV_{t}^{-1}}\\
% &\le 2\beta_{t}\norm{x_{it}}_{\bV_{t}^{-1}}\sqrt{P_{t}}
% \end{aligned}
% \]
% \end{proof}
% \end{lemma}

The following lemma states the confidence interval for each local cluster $(\ell, j)$, which is used by line~\ref{line:recommend} in \cref{alg:main}.
\BetaCi*
% \begin{lemma}\label{lemma:cdp_beta}
% Under the setting of FCLUB and fix a local cluster $j$ located at the server $\ell$ which shares the information with $L'\le L$ clusters (including itself), let the true preference vector be $\btheta^*$ and the true cluster be $j^*$, let $\hat{\btheta}^{\ell}_{t,j}=(\bS^{\ell}_{t,j})^{-1} u^{\ell}_{t,j}$.
% When all (global) clusters are correctly identified and partitioned, it holds with probability at least $1-2\alpha$,
% \begin{equation}
%     \norm{\btheta^*-\hat{\btheta}^{\ell}_{t,j}}_{\bS^{\ell}_{t,j}} \le \beta^{\ell}_{t,j}
% \end{equation}
% where $\beta^{\ell}_{t,j} \triangleq \beta^{\ell}_j(\tilde{T}^{\ell}_{t,j}, L, \alpha/(mL))= \sigma_{0}\sqrt{2\log(\frac{mL}{\alpha})+d\log(\frac{\rho_{\max}}{\rho_{\min}}+\frac{T^{\ell}_{t,j}}{dL'\rho_{\min}})} + \sqrt{L'\rho_{\max}} +\sqrt{L'}\kappa$
% \end{lemma}
% Let $tau$ be the time slot when the last upload or download event happens, and let $\bS^{\ell}_{t,j}=\bar{H}_{t,}\sum_{k=1}^{T^{\ell}_{t,j}}\bx_{t_k} \bx_{t_k}$

% Under the setting of CDP, let $\boldsymbol{\hat{\btheta}_{t}} = \argmin (\norm{\boldsymbol{X_{<t}}\boldsymbol{\btheta} - \boldsymbol{y_{<t}}}_{2}^{2}+\norm{\boldsymbol{\btheta}}_{\boldsymbol{H_{t}}^{2}})$, then $\norm{\boldsymbol{\btheta} - \boldsymbol{\hat{\btheta}_{t}}}_{\boldsymbol{\bV_{t}}} \le \beta_{t}$,  and $\beta_{t}$ should satisfy:
% \[
% \beta_{t} \le \sigma_{0}\sqrt{2\log(\frac{2}{\alpha})+d\log(3mL+\frac{t}{d\rho_{min}})} + S\sqrt{mL\rho_{\max}} +\sqrt{mL}\kappa_{t}
% \]
\begin{proof}
Let $\tau$ be the time slot when the last upload (or download) event happens before time $t$ and sequence $\{t_1, ..., t_{T^{\ell}_{\tau,j}}\}$ be the time slots when user $i_{t_i}$ that belongs to cluster $j^*$ appears.
Denote the $L'$ clusters that belongs to cluster $j^*$ to be $\{(\ell_1,j_1), ..., (\ell_{L'}, j_{L'})\}$.
With a little abuse of notation, we denote $\sum_{p=1}^{L'}\tilde{\bH}^{\ell_p}_{\tau, j_p}$ and $\sum_{p=1}^{L'}\bar{\bh}^{\ell_p}_{\tau, j_p}$ by $\tilde{\bH}_{\tau}$ and $\bar{\bh}_{\tau}$, respectively.
Recall that by definition, $\bS^{\ell}_{t,j}=\tilde{\bH}_{\tau} + \sum_{k=1}^{T^{\ell}_{t,j}}\bx_{t_k} \bx_{t_k}^{\top}, u^{\ell}_{t,j}=\bar{\bh}_{\tau} + \sum_{k=1}^{T^{\ell}_{\tau,j}}\bx_{t_k} y_{t_k}$.
Then we have 
\begin{align*}
    \btheta^{*}-\hat{\btheta}^{\ell}_{t,j}&=\btheta^{*}-(\bS^{\ell}_{t,j})^{-1}( \sum_{k=1}^{T^{\ell}_{\tau,j}}\bx_{t_k} y_{t_k} + \bar{\bh}_{\tau})\\
    &=\btheta^{*}-(\bS^{\ell}_{t,j})^{-1}( (\sum_{k=1}^{T^{\ell}_{\tau,j}}\bx_{t_k} \bx_{t_k}^{\top})\btheta^* + \sum_{k=1}^{T^{\ell}_{\tau,j}}\bx_{t_k} \eta_{t_k} + \bar{\bh}_{\tau})\\
    &=\btheta^{*}-(\bS^{\ell}_{t,j})^{-1}(\bS^{\ell}_{t,j} \btheta^* - \tilde{\bH}_{\tau} \btheta^*  + \sum_{k=1}^{T^{\ell}_{\tau,j}}\bx_{t_k} \eta_{t_k} + \bar{\bh}_{\tau})\\
    &=(\bS^{\ell}_{t,j})^{-1}( \tilde{\bH}_{\tau} \btheta^*  - \sum_{k=1}^{T^{\ell}_{\tau,j}}\bx_{t_k} \eta_{t_k} - \bar{\bh}_{\tau})
\end{align*}

Then we multiply the both sides by $(\bS^{\ell}_{t,j})^{1/2}$,
\begin{align}
    \norm{\btheta^{*}-\hat{\btheta}^{\ell}_{t,j}}_{\bS^{\ell}_{t,j}} &= \norm{\tilde{\bH}_{\tau} \btheta^*  - \sum_{k=1}^{T^{\ell}_{\tau,j}}\bx_{t_k} \eta_{t_k} - \bar{\bh}_{\tau}}_{(\bS^{\ell}_{t,j})^{-1}}\notag\\
    &\le \norm{\tilde{\bH}_{\tau} \btheta^*}_{(\bS^{\ell}_{t,j})^{-1}} + \norm{\sum_{k=1}^{T^{\ell}_{\tau,j}}\bx_{t_k} \eta_{t_k}}_{(\bS^{\ell}_{t,j})^{-1}} +  \norm{\bar{\bh}_{\tau}}_{(\bS^{\ell}_{t,j})^{-1}}\notag\\
    &\le \norm{\tilde{\bH}_{\tau} \btheta^*}_{\tilde{\bH}_{\tau}^{-1}} + \norm{\sum_{k=1}^{T^{\ell}_{\tau,j}}\bx_{t_k} \eta_{t_k}}_{\left(\sum_{k=1}^{T^{\ell}_{\tau,j}}\bx_{t_k} \bx_{t_k}^{\top}+ L' \rho_{\min}I\right)^{-1} } + \norm{\bar{\bh}_{\tau}}_{\tilde{\bH}_{\tau}^{-1}}\label{eq:apdx_beta_1}\\
    &\le  \sqrt{L'\rho_{\max}} + \sigma_{0}\sqrt{2\log(\frac{2}{\alpha})+d\log(\frac{\rho_{\max}}{\rho_{\min}}+\frac{T^{\ell}_{t,j}}{dL'\rho_{\min}})}  +\sqrt{L'}\kappa\label{eq:apdx_beta_2}
\end{align}
where \Cref{eq:apdx_beta_1} uses $\bS^{\ell}_{t,j}=\tilde{\bH}_{\tau}+\sum_{k=1}^{T^{\ell}_{\tau,j}}\bx_{t_k} \bx_{t_k}^{\top}$ so we have $\tilde{\bH}_{\tau} \preceq \bS^{\ell}_{t,j}$ and under event $B_0(\alpha)$ we have  $  \sum_{k=1}^{T^{\ell}_{\tau,j}}\bx_{t_k} \bx_{t_k}^{\top} + L' \rho_{\min}I \preceq \bS^{\ell}_{t,j}$, \Cref{eq:apdx_beta_2} uses the \Cref{def:accurate} and the self-normalized bound in \citep[Theorem 2]{abbasi2011improved}.

\end{proof}

\subsection{Putting All Together}\label{apdx_sec:put_together}
\mainRegret*

\begin{proof}
Since $\norm{x}\le 1, \norm{\btheta}\le 1$, by \cref{cdp_to} the regret before all global clusters are correctly identified is $2T_0(\alpha) \cdot 1$, with probability at least $1-6\alpha$. After $2T_0(\alpha)$, by \cref{apdx_eq:prop1} the regret is upper bounded by $\tilde{O}\Big(dL\sqrt{mT\frac{\log(1/\delta)}{\varepsilon}}\log^{1.5} T \Big)$, with probability $1-2\alpha$. Finally, it suffices to set $\alpha = 1/(8T)$ and by union bounds, the above theorem holds with probability at least $1-1/T$.
\end{proof}

\section{Communication cost analysis}\label{sec:communication_cost}
In this section, we show how to upper bound the number of communication rounds before and after the cluster structures are detected correctly.

\begin{proposition}\label{communication cost}
Under the CDP setting, the total communication cost satisfies:
\begin{equation}
    C(T) \le mL\left(\log T + \frac{d\log(\frac{\rho_{\max}}{\rho_{\min}}+\frac{T}{d\rho_{\min}})}{ \log(\min\{U,D\})} +\frac{d\log(\frac{\rho_{\max}}{\rho_{\min}}+\frac{2T_0}{d\rho_{\min}})\log(2T_0)}{ \log(\min\{U,D\})}\right)
\end{equation}
\end{proposition}

\begin{proof}
We consider two cases, when $t > 2T_0$ and $t\le 2T_0$.

\textbf{Case 1: $t > 2T_0$.}

Recall that when $t > 2 T_0$, the cluster structure are correct and there will be no changing of clusters afterwards.
Fix any local cluster $j$ at local server $\ell$, suppose it shares information with total $L'$ clusters (including itself).
Let $\{t_{1}, t_{2}, ...t_{n'}\}$ be the sequence of time when upload event or download event happen and $\{\bS^{\ell}_{t_1, j}, \bS^{\ell}_{t_2, j}, ...,\bS^{\ell}_{t_n',j}\}$ be the corresponding sequence of local Gram matrix after each event happens. Similar to the proof of 

Then we have:
\begin{equation}
    \log(\frac{\det (\bS^{\ell}_{t_n',j})}{\det(\bS^{\ell}_{t_1, j}) }) = \log(\frac{\det(\bS^{\ell}_{t_2, j})}{\det(\bS^{\ell}_{t_1, j})}) + \log(\frac{\det(\bS^{\ell}_{t_3, j})}{\det(\bS^{\ell}_{t_2, j})}) + ... + \log(\frac{\det(\bS^{\ell}_{t_{n'}, j})}{\det(\bS^{\ell}_{t_{n'-1}, j})}) \le \log(\frac{\det (\bV_{T})}{\det (L'\rho_{0, \min}I)})
\end{equation}

On one hand, by the definition of the upload and download event, we have $\log(\frac{\det(\bS^{\ell}_{t_{k+1}, j})}{\det(\bS^{\ell}_{t_k, j})}) \ge \log(\min\{U,D\})$ for any $k \in [n'-1]$.
On the other hand, by \citep[Lemma 22]{shariff2018differentially}, $\log(\frac{\det (\bV_{T,j^*})}{\det (L'\rho_{\min}I)}) \le d\log((L'd\rho_{\max}+T)/d)-d\log(L'\rho_{\min})=d\log(\frac{\rho_{\max}}{\rho_{\min}}+\frac{T}{dL'\rho_{\min}})\le d\log(\frac{\rho_{\max}}{\rho_{\min}}+\frac{T}{d\rho_{\min}})$, where the last inequality uses the fact that $L'\ge 1$.
So $n' \le \frac{d\log(\frac{\rho_{\max}}{\rho_{\min}}+\frac{T}{d\rho_{\min}})}{ \log(\min\{U,D\})}$.
Since we at most have $mL$ local clusters, the total communication is $ \frac{mLd\log(\frac{\rho_{\max}}{\rho_{\min}}+\frac{T}{d\rho_{\min}})}{ \log(\min\{U,D\})}$.

\textbf{Case 2: $t < 2T_0$.}
Equivalently, we are in phase $s \le \lceil \log T_0 \rceil$, then each phase we will communicate $\frac{mLd\log(\frac{\rho_{\max}}{\rho_{\min}}+\frac{2T_0}{d\rho_{\min}})}{ \log(\min\{U,D\})}$ using the similar argument.
Then the total communication is $\frac{mLd\log(\frac{\rho_{\max}}{\rho_{\min}}+\frac{2T_0}{d\rho_{\min}})\log(2T_0)}{ \log(\min\{U,D\})}$.

For the communication during the beginning of each phase in order to detect the cluster structure, we know there will be at most $mL \log T$ communications, which concludes the lemma.

% Fix any local cluster $j$ at local server $i$, and define the sequence of time when the event "upload" and "download" happen as $\left\{t_{1}, t_{2}, ...t_{C_{T},i}\right\}$. Then the corresponding sequence of local covariance matrices are $\left\{\rho_{min}I, \bV_{i,t_{1}}, \bV_{i,t_{2}},...\bV_{i,t_{C_{T},i}}\right\}$. Then we have:
% \[
% \log(\frac{\det \bV_{i,t_{C_{T},i}}}{\det \rho_{min}I }) = \log(\frac{\det \bV_{i,t_{1}}}{\det \rho_{min}I}) + \log(\frac{\det \bV_{i,t_{2}}}{\det \bV_{i,t_{1}}}) + \log(\frac{\det \bV_{i,t_{3}}}{\det \bV_{i,t_{2}}}) + ... + \log(\frac{\det \bV_{i,t_{C_{T},i}}}{\det \bV_{i,t_{,i-1}}}) \le \log(\frac{\det \bV_{T-1}}{\det \rho_{min}I})
% \]
% As all these terms in the summation is lower bounded by $\log \min(\gamma_{u}, \gamma_{D})$, the communication cost for server $i$ should satisfy $C_{T,i} \le \frac{\log(det \bV_{T-1})-d\log(\rho_{min})}{\log(\min(\gamma_{u},\gamma_{D}))}$, and $\det \bV_{T-1} \le (L\rho_{\max}+T/d)^{d} $. Consider the upload happens at the  beginning of each phase, then the total communication cost of $L$ servers should be 
% \[C_{T} \le \sum_{i=1}^{L}\frac{d\log T}{\log min(\gamma_{u},\gamma_{D})} +  Lm\]
\end{proof}

\section{Privacy Guarantee}\label{sec:privacy_proof}
To guarantee the algorithm preserves the CDP, we use a tree based algorithm where each tree node preserves $(\varepsilon/\sqrt{8\nu \ln(2/\delta)}, \delta/2)$-DP. For each local cluster $(\ell,j)$, the local privatizer PVT($\ell,j$) is $(\varepsilon, \delta)$-DP and together they preserve $(\varepsilon, \delta, L, m)$-CDP as in \cref{def:cdp}. Also note that our protocol only uploads the protected clustered data rather than each user's individual data, which satisfies the privacy requirements in \cref{sec: problem setting}.
% This requires $\rho_{\max}=$.

\section{Experiment Settings and Supplemental Experiments}\label{sec:Supplemental_Experiments}
\subsection{Experiment Settings}\label{sec:exp_setting}
\textbf{Synthetic Dataset.}
We consider a setting of $n=40$ users with $m = 4$ global clusters and $L = 5$ local servers. The weight vectors and item vectors are generated randomly with dimension $d = 10$. For each global cluster $j\in[m]$, we fix a weight vector $\btheta_j$ and we guarantee that all weight vectors are orthogonal and the difference gap $\gamma$ between them is $\sqrt{2}$. We stipulate that the default privacy budget is $(\epsilon,\delta)=(1,0.1)$ and communication upload/download rate $U=D=1.01$. In each round, a random user comes and the algorithm selects one of $K=10$ items to recommend to the user. After we receive the feedback and compare it with the best action based on the true weight vectors, we evaluate the performance of the algorithm by calculating the cumulative regret.

\textbf{MovieLens Dataset.} We randomly draw $10^3$ movies and $10^3$ users as preparation for the experiment. Similar to \citet{li2018online}, we generate the weight vectors on the basis of the $10^3$ users and the movies they rated. We randomly select $40$ users from these $10^3$ users for the experiment. In this experiment, we assume that there are $n=40$ users with $m = 4$ global clusters and $L = 5$ local servers and the weight vectors and item vectors are generated randomly with dimension $d = 10$. The users are selected randomly from these $10^3$ users for the experiment. Besides, we stipulate that the default privacy budget is $(\epsilon,\delta)=(1,0.1)$ and communication upload/download rate $U=D=1.01$.
\subsection{Parameter Study}\label{sec:exp_pstudy}
We conduct the parameter study on a synthetic dataset. Our goal in this section is to numerically investigate the influence of parameters' change on the performance of our CDP-FCLUB-DC algorithm. 
\begin{figure}[H]
 \centering
 \begin{subfigure}[b]{0.435\textwidth}
  \centering
  \includegraphics[width=\textwidth]{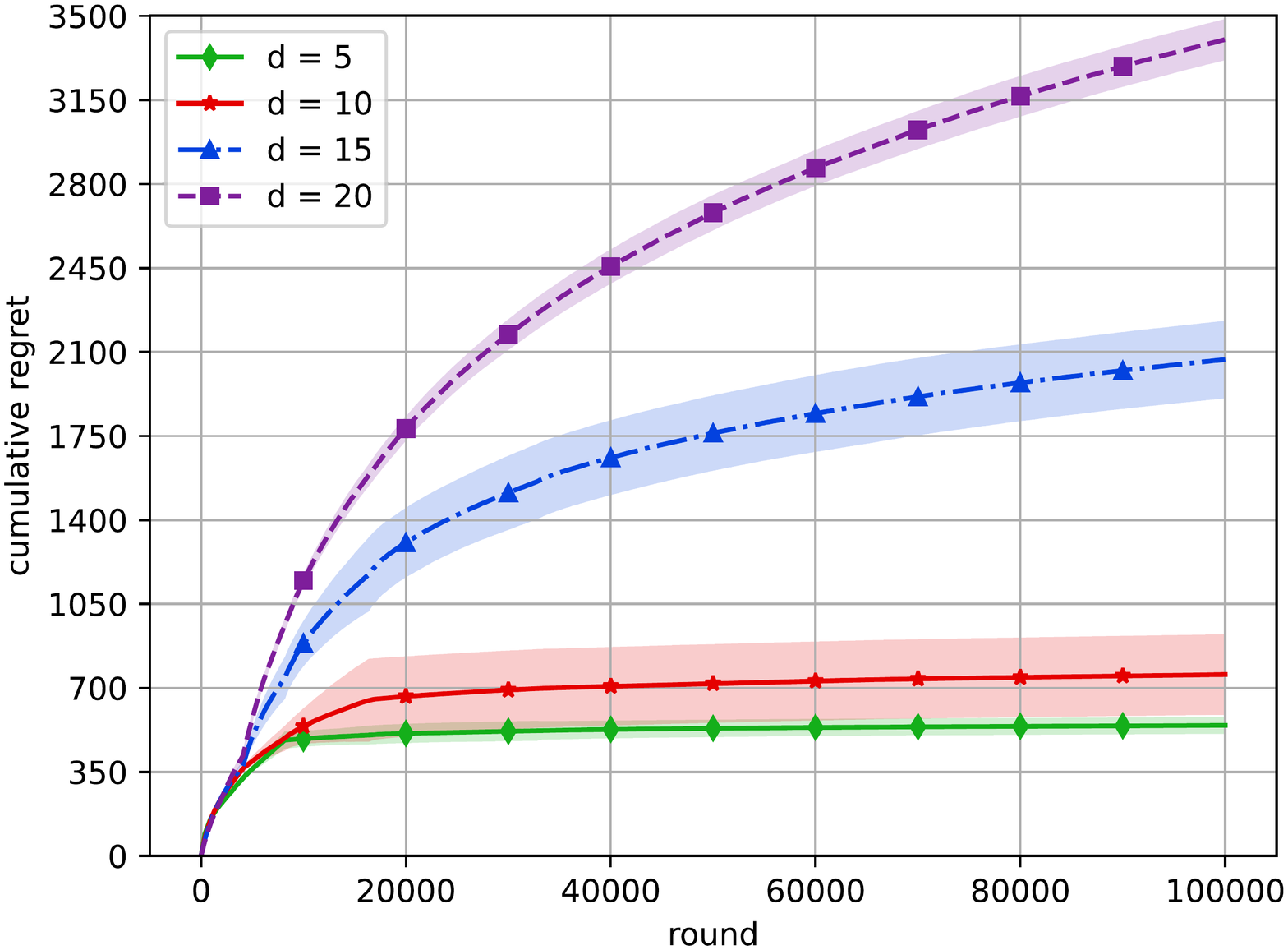}
  \caption{Vary Dimension $d$ and Fix Other Variables }
  \label{fig: d}
 \end{subfigure}
 \hfill
 \begin{subfigure}[b]{0.435\textwidth}
  \centering
  \includegraphics[width=\textwidth]{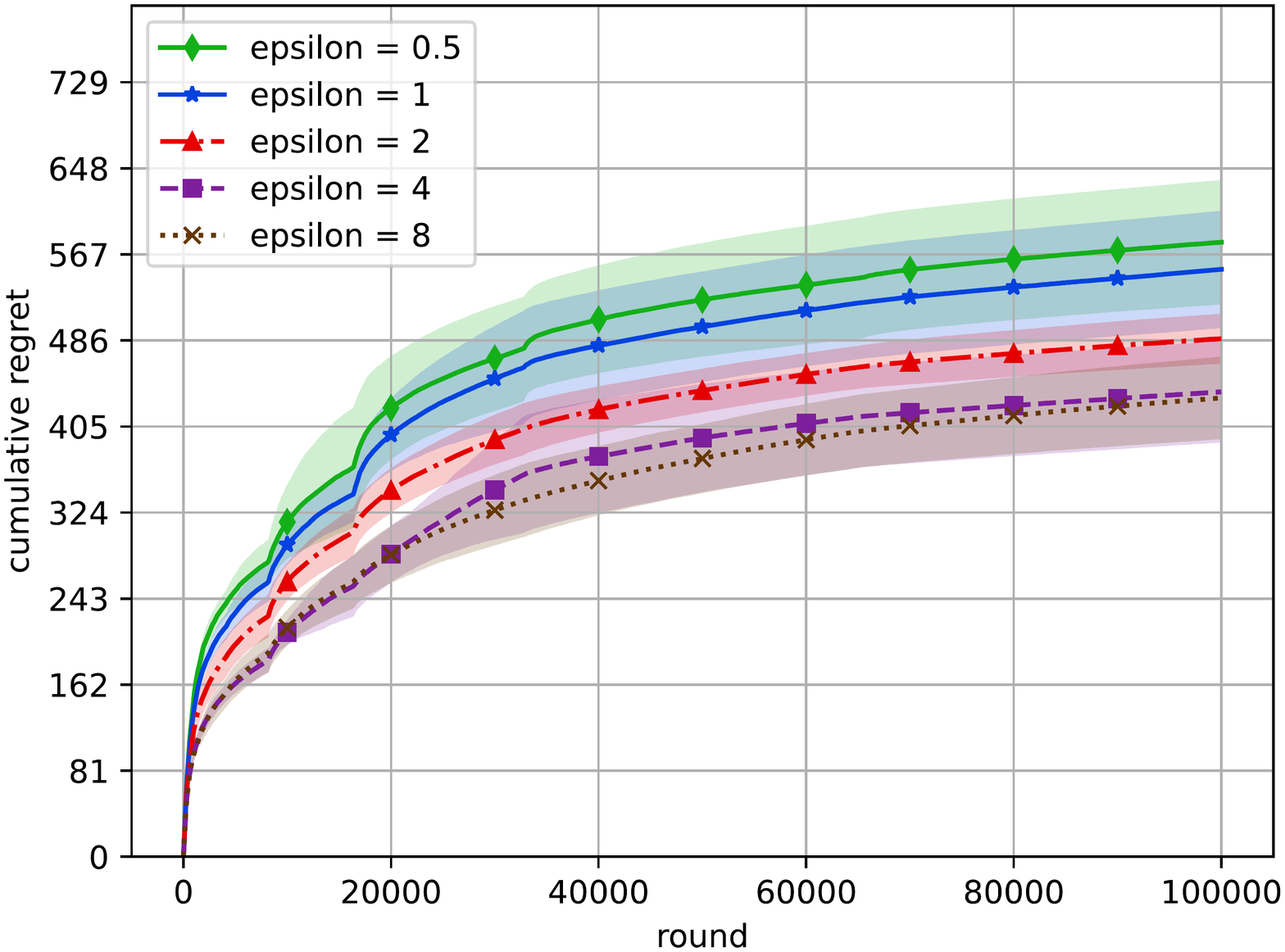}
  \caption{Vary Privacy Budget $\epsilon$ and Fix Other Variables}
  \label{fig: epsilon}
 \end{subfigure}
 \begin{subfigure}[b]{0.435\textwidth}
  \centering
  \includegraphics[width=\textwidth]{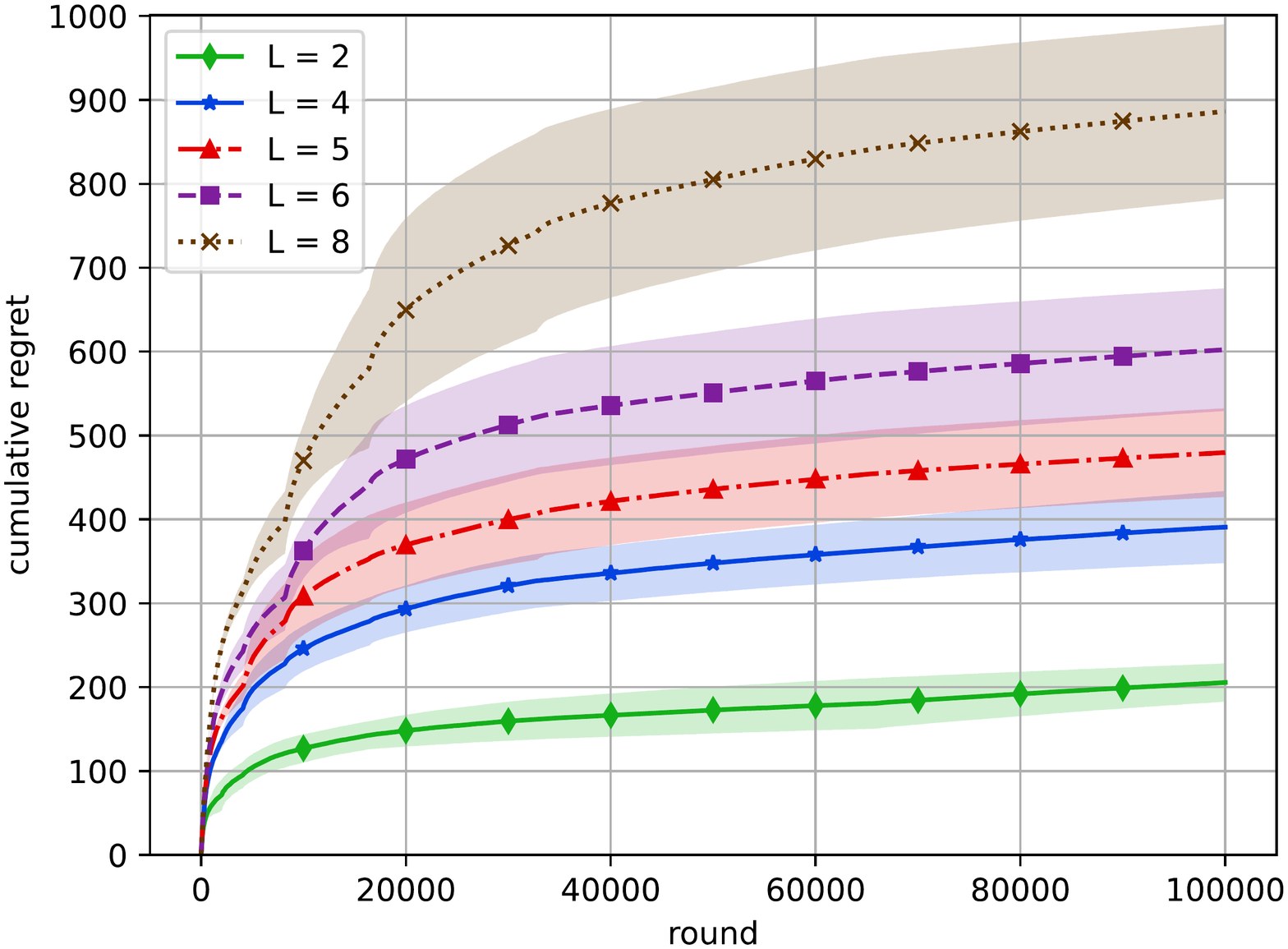}
  \caption{Vary the Number of Local Servers $L$ and Fix Other Variables}
  \label{fig:L}
 \end{subfigure}
 \hfill
 \begin{subfigure}[b]{0.435\textwidth}
  \centering
  \includegraphics[width=\textwidth]{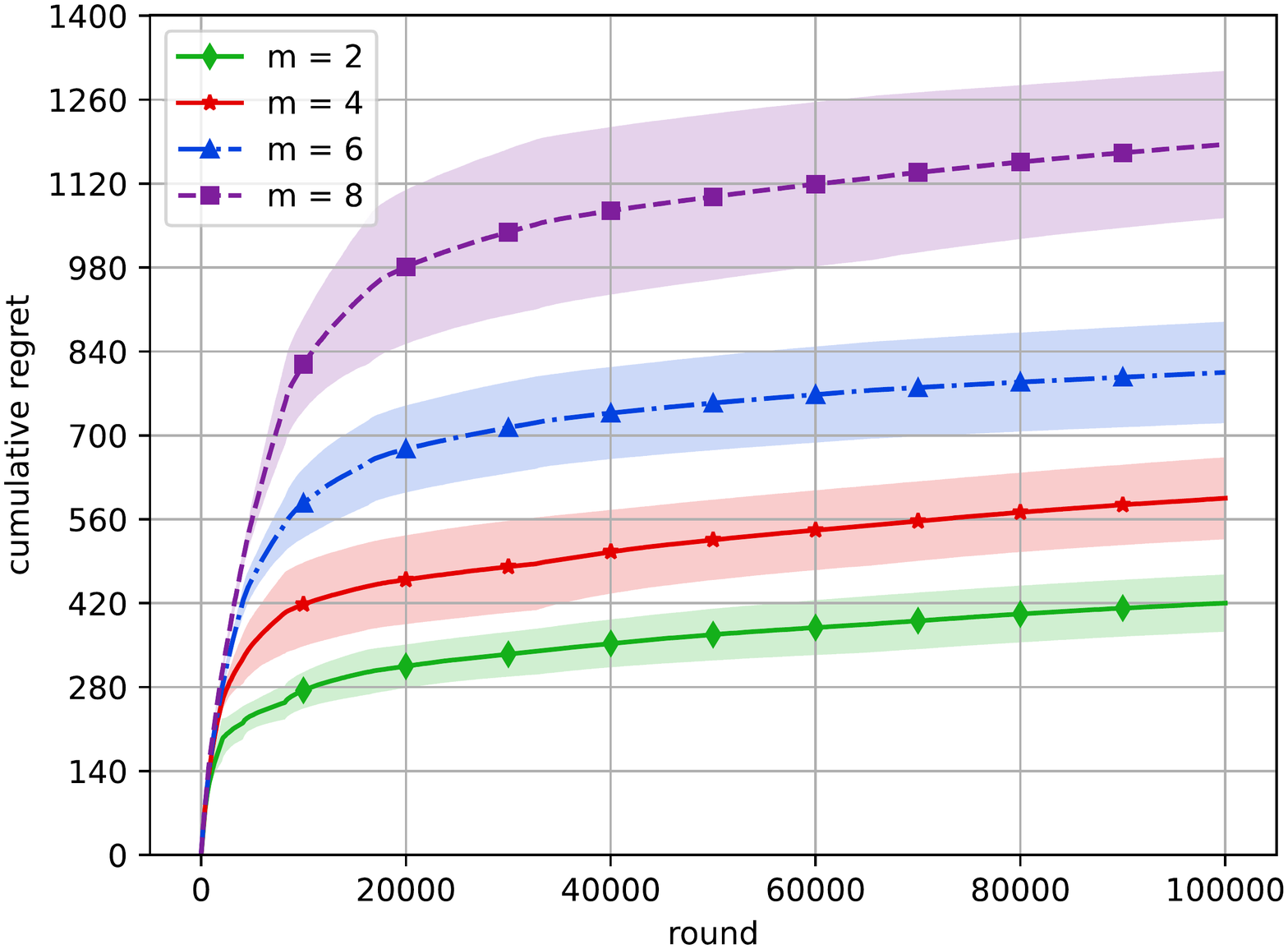}
  \caption{Vary the Number of Global Cluster $m$ and Fix Other Variables}
  \label{fig: m}
 \end{subfigure}
  \begin{subfigure}[b]{0.435\textwidth}
  \centering
  \includegraphics[width=\textwidth]{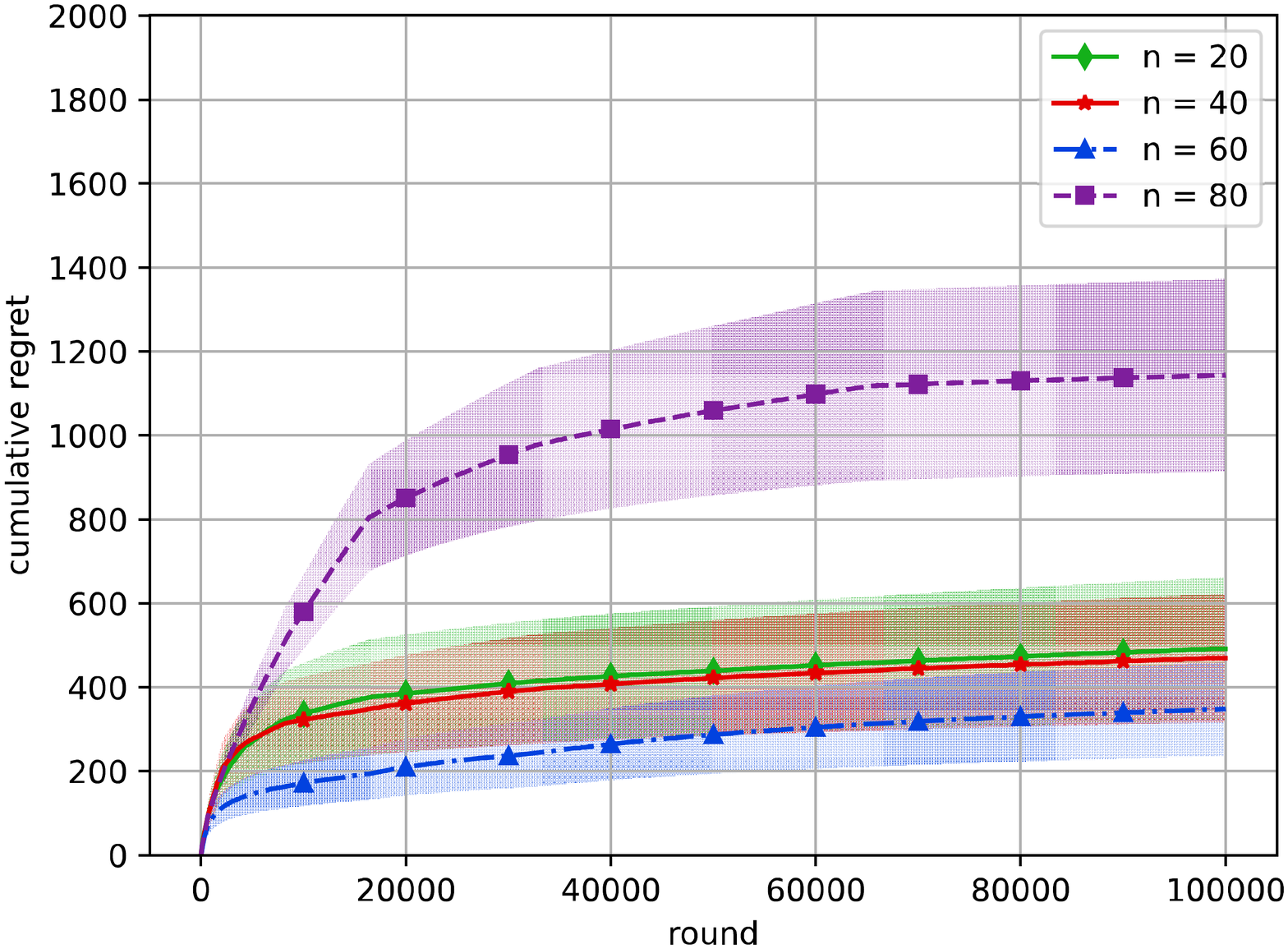}
  \caption{Vary the Number of users $n$ and Fix Other Variables}
  \label{fig: n}
 \end{subfigure}
 \caption{Parameter Study on the Synthetic Dataset}\label{fig:compare1}
\end{figure}
In Figure \ref{fig: d}, we vary the dimension $d$ from 5 to 20 and fix other parameters. It can be seen that with the increase of dimension $d$, the convergence speed of regret will slow down. In Figure \ref{fig: epsilon}, we change the privacy budget $\epsilon$ from 0.5 to 8. It is obvious that when $\epsilon$ is small, the cumulative regret is better because at this time the privacy will have less impact on the recommendation. In Figure \ref{fig:L} and Figure \ref{fig: m}, we vary the number of local server $L$ from 2 to 8 and the number of global cluster $m$ from 2 to 8 respectively. The results show that the value of $L$ and $m$ are positively correlated with the cumulative regret, which is consistent with our conclusion in Theorem \ref{thm:mainRegret}. For the number of users $n$, Figure \ref{fig: n} shows that the empirical results deviate from \Cref{thm:mainRegret} since the regret should be larger when there are more users. As we can see that the shaded area almost overlapped for $n=20,40,60$, we conjecture the derivation comes from the randomness of our algorithm and these curves should behave normally when we conduct more independent experiments.

\subsection{communication cost}\label{sec:exp_communication_cost}
\begin{figure}[H]  
\centering
\begin{subfigure}[b]{0.465\textwidth}
\centering
\includegraphics[width=\textwidth]{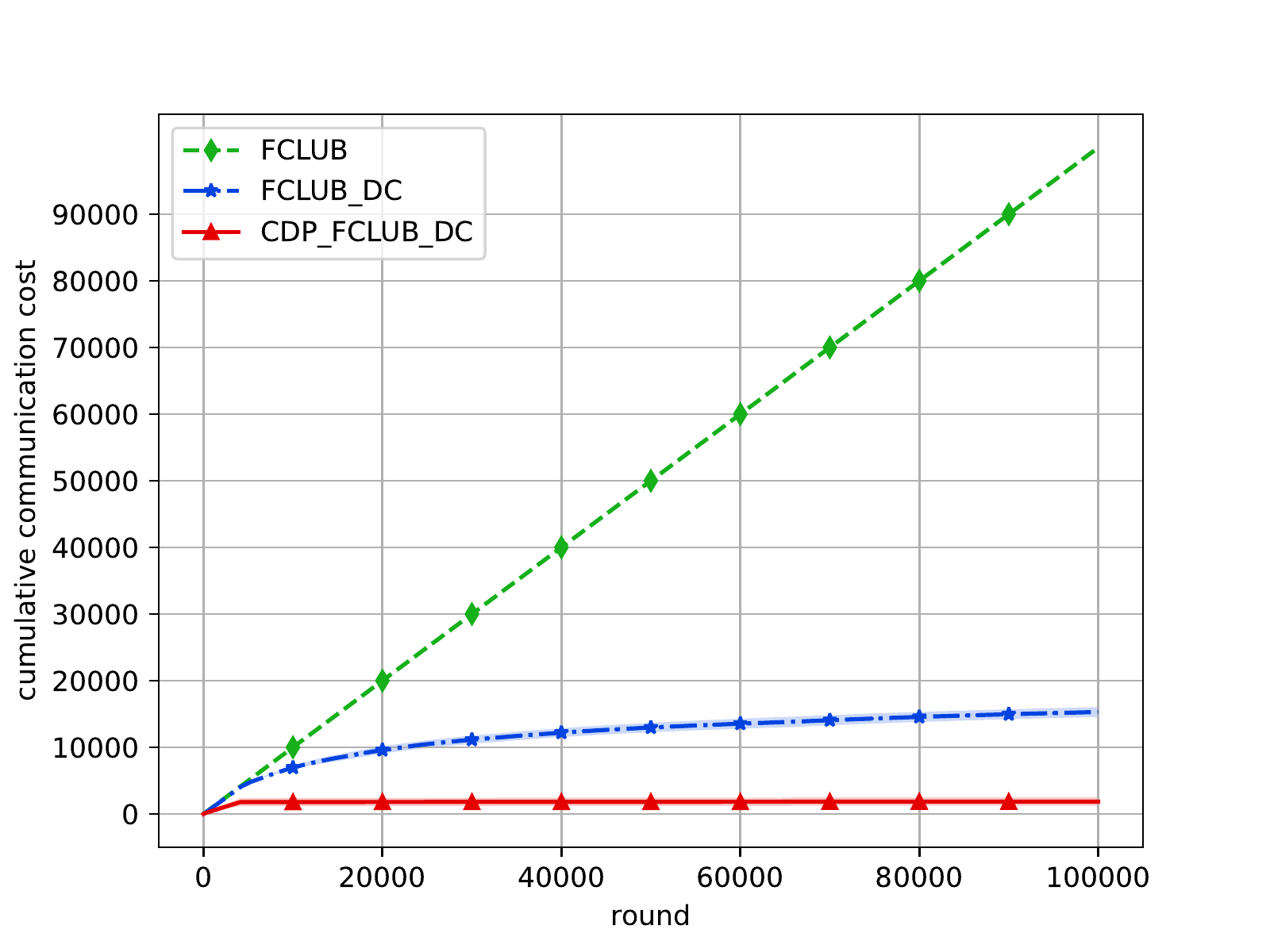}
\caption{Communication Cost for the Synthetic Dataset}
\end{subfigure}
\begin{subfigure}[b]{0.465\textwidth}
\centering
\includegraphics[width=\textwidth]{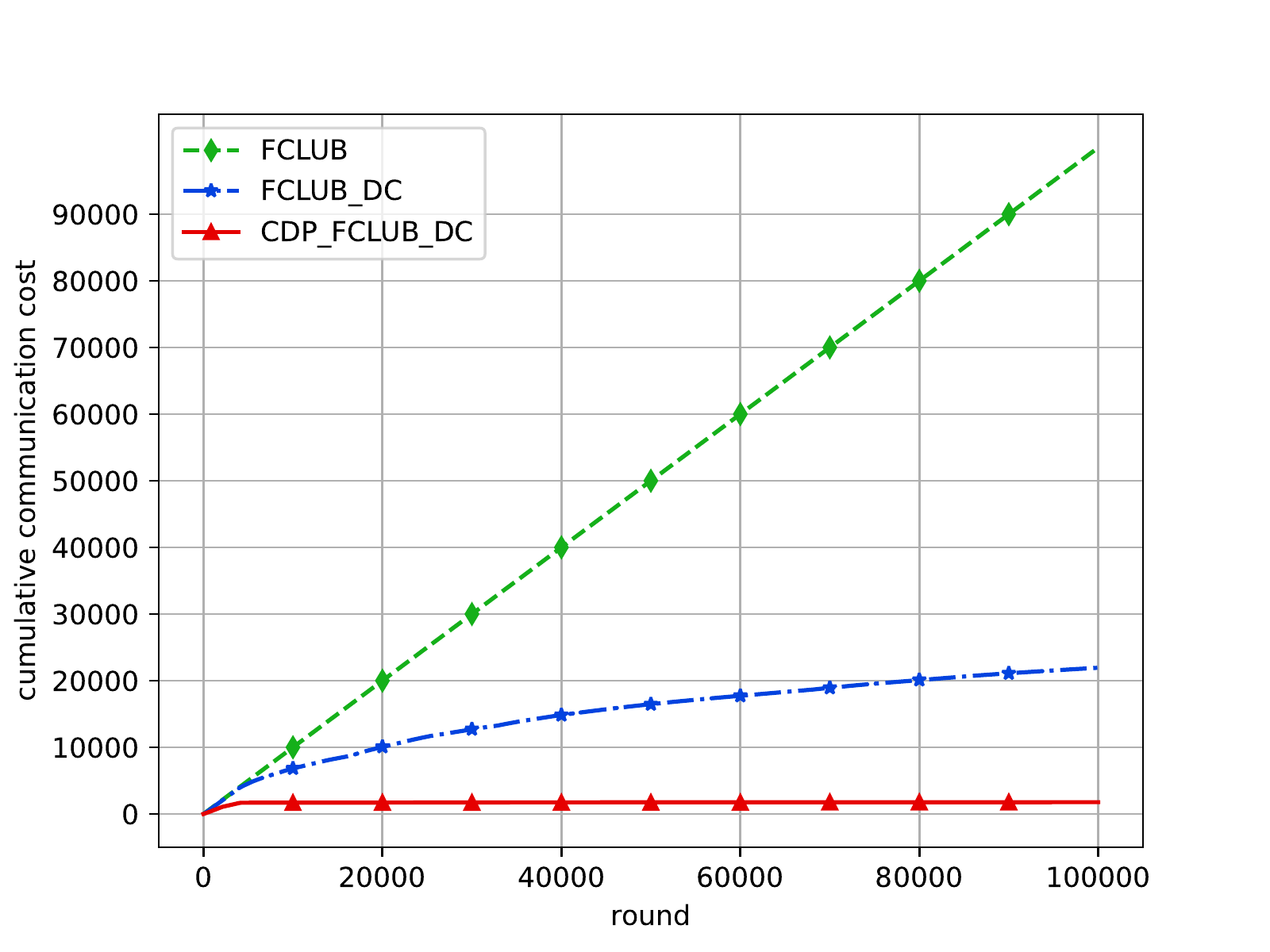}
\caption{Communication Cost for the MovieLens Dataset}
\end{subfigure}
\caption{Comparative Experiments on the Communication Costs}
\label{fig:communication costs}
\end{figure}
Communication cost is one of the most important factors we consider when designing the algorithm. In Figure \ref{fig:communication costs}, we compare the communication cost of FCLUB using full communication with that of FCLUB-DC and CDP-FCLUB-DC using the delayed communication. The results clearly indicate that our asynchronous communication protocol can reduce the communication cost effectively.
\subsection{Running Time}\label{sec:exp_running_time}
\begin{table}[!h]
\centering
\caption{Comparison of the Running Time}
\begin{tabular}{c|cccccccc}
\hline
          & CDP-FCLUB-DC\;& FCLUB-DC\;& FCLUB\;& SCLUB\;& CLUB\;& LinUCB\;& Homo-DC\;& Homo\\  
\hline          
run time (ms)      & 1.707
\;  & 1.814  \; & 58.870
\;  & 3.805 \; \;  & 0.779
\; & 0.647   \; & 1.277
\; & 29.265
\\
run time (ms)  &1.667\; &1.892\; &83.492\; &3.855 \; &0.776 \; &0.654 \; &1.301 \; &35.645\\ 
\hline
\end{tabular}
\label{tab:runtime}
\end{table}
We also compare the average running time (ms) of each round between our algorithm CDP-FCLUB-DC and baselines on synthetic dataset and MovieLens dataset (the first line is synthetic dataset and the second line is MovieLens dataset). It can be seen from Table \ref{tab:runtime} that due to the use of delay communication, our algorithm has a great improvement in communication cost compared with FCLUB and our run time cost is even lower than SCLUB. However, because of the existence of communication, our run time cost is still higher than CLUB and LinUCB.

\end{document}